\documentclass[11pt]{article}

\PassOptionsToPackage{hyperfootnotes=false}{hyperref}
\usepackage[margin=1in]{geometry}
\usepackage{amsmath,amssymb,amsfonts,mathtools,amsthm}
\usepackage{bm}
\usepackage{comment}
\usepackage{enumitem}
\usepackage{microtype}
\usepackage{array}
\usepackage{booktabs}
\usepackage{tabularx}
\usepackage{makecell}
\usepackage{threeparttable}
\usepackage{longtable}
\usepackage[table]{xcolor}
\usepackage{placeins}
\usepackage{float}
\usepackage{graphicx}
\usepackage{natbib}
\usepackage{authblk}
\usepackage[normalem]{ulem}
\usepackage[colorlinks=true,citecolor=blue,linkcolor=blue,urlcolor=blue]{hyperref}

\allowdisplaybreaks[2]
\setlist{topsep=3pt,itemsep=2pt,parsep=1pt}
\newcolumntype{L}[1]{>{\raggedright\arraybackslash}p{#1}}
\newcolumntype{C}[1]{>{\centering\arraybackslash}p{#1}}
\newcolumntype{Y}{>{\raggedright\arraybackslash}X}
\newcommand{\WD}{\mathrm{WD}}

\DeclareMathOperator{\Rank}{rank}
\DeclareMathOperator{\SRank}{srank}

\DeclareMathOperator{\diag}{diag}
\DeclareMathOperator{\Diag}{Diag}
\DeclareMathOperator{\softmax}{softmax}

\numberwithin{equation}{section}
\renewcommand{\paragraph}[1]{\par\smallskip\noindent\textit{#1.}\ }

\theoremstyle{plain}
\newtheorem{theorem}{Theorem}[section]
\newtheorem{proposition}[theorem]{Proposition}
\newtheorem{lemma}[theorem]{Lemma}
\newtheorem{corollary}[theorem]{Corollary}
\theoremstyle{definition}
\newtheorem{definition}[theorem]{Definition}
\newtheorem{assumption}[theorem]{Assumption}
\theoremstyle{remark}

\newenvironment{proofidea}{\par\noindent\textbf{Proof idea.}\ }{\hfill$\blacksquare$\par\smallskip}
\newcommand{\proofat}[1]{\par\noindent\textit{Proof.} Deferred to Appendix~\ref{#1}.\par\smallskip}

\begin{document}

\title{Deciphering Two Training Clocks in Grokking via Deep Linear Network Theory with Conditional ReLU Reduction}

\author[1,2]{Hu Tan}
\author[3]{Kuo Gai}
\author[1,2,4,*]{Shihua Zhang}

\affil[1]{State Key Laboratory of Mathematical Sciences, Academy of Mathematics and Systems Science, Chinese Academy of Sciences, Beijing 100190, China}
\affil[2]{School of Mathematical Sciences, University of Chinese Academy of Sciences, Beijing 100049, China}
\affil[3]{Shanghai Institute for Mathematics and Interdisciplinary Sciences (SIMIS), Shanghai, China}
\affil[4]{Key Laboratory of Systems Health Science of Zhejiang Province, School of Life Science, Hangzhou Institute for Advanced Study, University of Chinese Academy of Sciences, Chinese Academy of Sciences, Hangzhou 310024, China}
\affil[ ]{
%Emails: \texttt{tanhu2020@amss.ac.cn}; \texttt{gaikuo@simis.cn}; 
$^*$Corresponding author: \texttt{zsh@amss.ac.cn}}

\date{}

\maketitle
\vspace{-1.6em}

\begin{abstract}
Grokking suggests that fitting the training data and learning a simple underlying rule may occur on different time scales. We formalize this phenomenon by separating the fast decay of the classification loss from the slower simplification of the learned representation, and we call the resulting pair of stopping times \emph{two training clocks}. For deep linear networks, we show that a post-margin gap-growth or one-step tail-contraction condition reduces the cross-entropy loss to level $\varepsilon$ on a logarithmic time scale. In contrast, when layerwise weight decay is present, the induced regularization on the end-to-end map can be expressed as a Schatten-type penalty; under a sharp late-time Kurdyka--\L{}ojasiewicz tail, this structural energy closes on a polynomial time scale. The two clocks, therefore, separate fitting from representation simplification. We then explain how the same mechanism can appear in ReLU MLPs. In regions where the activation patterns on the training set remain fixed, the network reduces to a linear model in the active coordinates. In a two-layer ReLU embedding model, chain-rule estimates further show that the classifier head can receive larger effective gradients than the embedding block under controlled downstream norms. This supports a two-stage mechanism in which the classifier fits first, while the representation continues to simplify later. We use modular addition as the main experimental setting. The deep linear theory provides the rigorous core of the analysis. But the ReLU results are formulated as conditional reductions that account for empirical behavior without claiming a global proof for nonlinear training dynamics.
\end{abstract}
\vspace{0.1em}


\vspace{0.9em}
%\noindent{\large\bfseries Significance Statement}\par\smallskip

\noindent\textbf{Significance Statement:} Deep neural networks often appear to have finished learning when their training error is already very small. Yet in grokking, a model may only later discover a simpler rule that works on new data. This paper explains this gap by separating two parts of training: fitting the observed labels and reorganizing the internal representation. We demonstrate in a model that can be mathematically analyzed that these two processes can run on very different time scales. Experiments on modular arithmetic support the same picture in nonlinear networks. The work suggests that extra training can matter because it continues to reshape what the model has learned, not just because it reduces error.

\section{Introduction}
\label{introduction}
In many grokking experiments, a model first reaches nearly perfect training accuracy and only much later reaches high test accuracy. The empirical loss may already be small during this long intermediate period, so the delayed improvement calls for a more precise question: after the training loss has essentially vanished, what is continued training still changing?

Test accuracy records the final predictive outcome, but it compresses two different processes: fitting a classifier to the observed labels and reshaping the representation into a rule-aligned form. The central claim is that these processes run on different training clocks. Classifier fitting can reach effective completion early, while the representation continues to evolve along a slower structural trajectory.

We formalize this separation through the object we call \emph{two training clocks}. The first clock is the \emph{classifier clock}: it records the time needed for the effective classifier to drive the cross-entropy tail below a target level. The second clock is the \emph{representation clock}: it records the time needed for the learned map to simplify, measured by a spectral penalty, a stable-rank proxy, or an equivalent low-dimensionality statistic. Grokking is then interpreted through the mismatch between these two clocks.

Across modular arithmetic, synthetic rule-learning tasks, and small natural-data settings, delayed generalization is most visible when the model can fit the training set early while regularization, initialization, or optimizer bias continues to reshape its internal geometry \citep{power2022grokking,JMLR:v25:22-1228,huang2024unified}; related evidence appears in \citet{liu2022towards} and \citet{liu2022omnigrok}. This recurring pattern points to the temporal separation between fitting and simplification as the main object to isolate (see representative settings in Table~\ref{tab:grokking}).

We analyze deep linear networks trained with cross-entropy loss and layerwise weight decay. This model keeps the three quantities needed for the argument in the same mathematical frame: margins for classification, singular values for representation geometry, and the Schatten-type penalty induced by layerwise weight decay. It discards the global switching behavior of ReLU gates, but it retains the spectral degrees of freedom through which low-rank structure can emerge.

The first result describes the fast clock. Once the effective classifier is in a post-margin regime where the incorrect-class logits keep separating at a definite rate, the cross-entropy tail decreases exponentially, and the time to reach loss level $\varepsilon$ is logarithmic in $1/\varepsilon$. We state this condition explicitly, avoiding a PL shortcut: a fixed positive margin by itself falls short of a uniform PL inequality for cross-entropy, because the softmax Hessian degenerates as the margin grows.

The second result describes the slow clock. In a deep linear network, layerwise weight decay corresponds to a Schatten-$p$ penalty on the end-to-end map, with $p=2/L$. On a late-time smooth spectral stratum, this penalty applies stronger shrinkage to small singular values and therefore biases the representation toward low effective rank. Under a sharp KL-tail condition near the limiting point, the associated energy gap decays on a polynomial clock. This polynomial clock is the quantitative counterpart of the slow representation drift observed during grokking.

\begin{figure}[t!]
    \centering
    \includegraphics[width=0.75\linewidth]{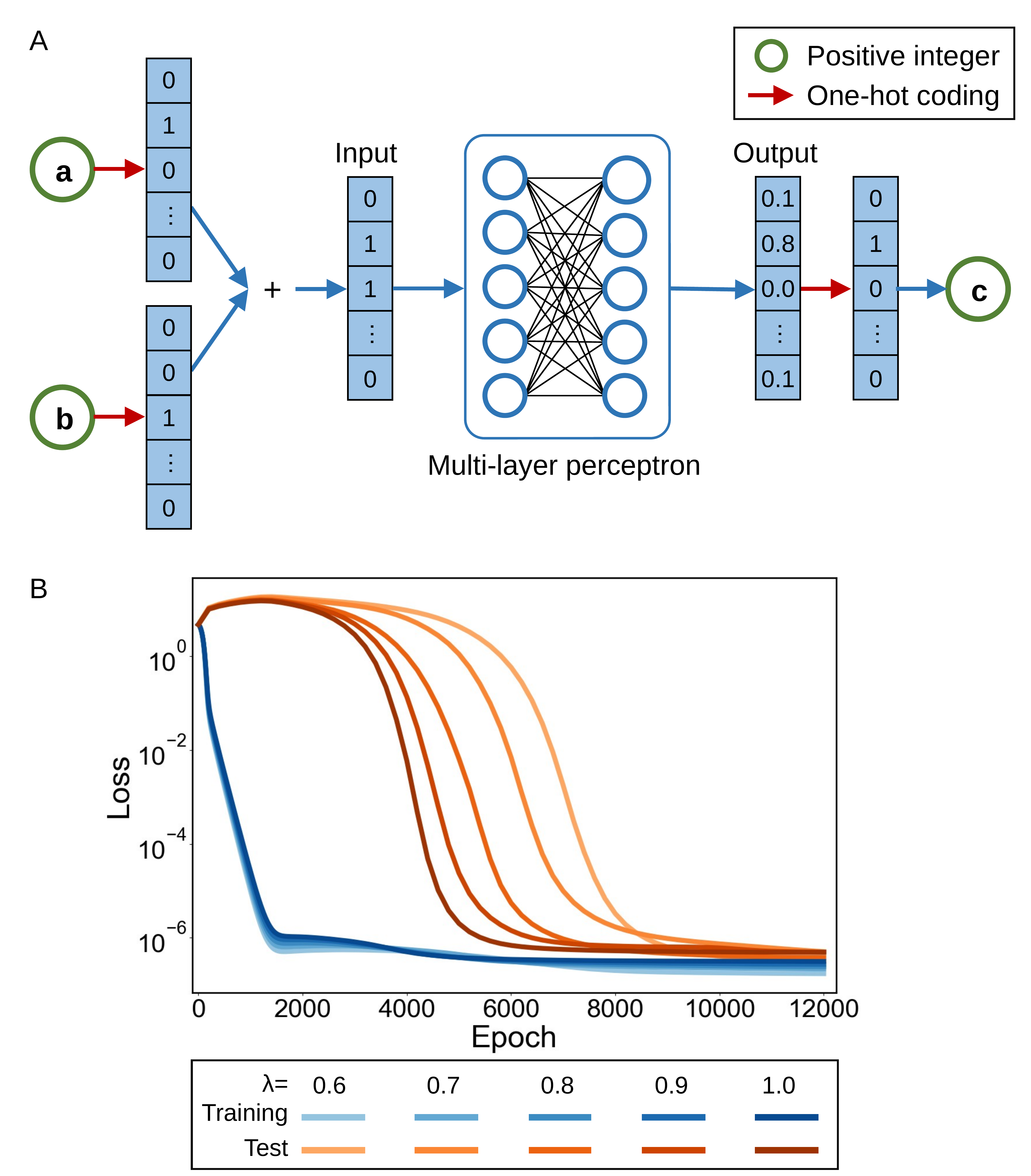}
    \caption{Representative training and test loss curves for a ReLU MLP on modular addition (mod $113$), together with the network architecture. (a)~Architecture sketch: each input pair $(a, b)$ is represented by trainable token embeddings, combined symmetrically, and processed by an MLP with ReLU activations and a softmax layer. (b)~Training and test loss curves for weight decay values $0.6$, $0.7$, $0.8$, $0.9$, and $1.0$ (blue: training; orange: test). The experiments are finite-width and nonlinear; throughout the paper, they serve as qualitative evidence for the time-scale-separation mechanism, while the deep-linear theorems provide the idealized core.}
    \label{fig:grokking-mlp}
\end{figure}

The ReLU part of the paper provides a conditional transfer principle for stabilized nonlinear training windows. Once activation patterns on the training set stabilize, a finite ReLU network becomes an active linear subsystem on those samples. In addition, for a two-layer embedding model, the gradient entering the embedding block carries extra downstream operator-norm factors compared with the classifier-head gradient. These two facts explain how the linear two-clock mechanism can become relevant to the nonlinear modular-addition experiments: the head can fit first, while the active representation continues to simplify later.

\subsection{Main Findings and Scope}
\label{subsec:contrib}
The main theorem-level claim is a time comparison (Table~\ref{tab:claim-scope}). The classifier clock is logarithmic under an explicit post-margin gap-growth or tail-contraction condition. The representation clock is polynomial under a late-time sharp KL-tail condition for the regularized spectral energy. These two statements are proved in the deep-linear core model, where the relevant objects are directly measurable and the assumptions can be stated without hidden nonlinear effects.

The ReLU statements have a different status. The frozen-gate proposition is exact on any time interval in which the activation masks are fixed on the training set. The gradient-hierarchy proposition is also exact as a chain-rule identity, with the head-dominance interpretation requiring bounded downstream norms and nonvanishing hidden features. What remains conditional is the claim that a particular ReLU training run enters such a stabilized active regime. %Table~\ref{tab:claim-scope} summarizes the resulting claim levels.

\subsection{Related Work}
\label{subsec:related-work}
\paragraph{Grokking as delayed generalization}
The phenomenon was first highlighted in algorithmic tasks such as modular arithmetic \citep{power2022grokking}. Subsequent work showed that the delay can be shortened, amplified, or removed by changes in regularization, data size, initialization, and optimization \citep{JMLR:v25:22-1228,varma2023explaining,clauw2024information,huang2024unified}; see also \citet{liu2022towards} and \citet{liu2022omnigrok}. This line of work focuses on the empirical delay between training and test performance; the present paper isolates a finer dynamical object inside that delay, namely the separation between the fitting clock and the representation clock.

\paragraph{Feature learning and embeddings}
Several accounts of grokking emphasize that generalization appears only after feature movement beyond a kernel-like or lazy-training regime \citep{chizat2019lazy,yang2020feature,woodworth2020kernel,kumar2024grokking,mohamadi2024you}. Mechanistic studies of modular arithmetic further identify Fourier, cyclic, circuit-efficiency, or geometric structure that emerges late in training \citep{nanda2023progress,varma2023explaining,musat2024understanding,wang2024progressive,aubry2024transformer}. Recent embedding-centric work shows that useful embeddings can transfer across runs and that the embedding layer itself can be a bottleneck for delayed generalization \citep{xu2025groktransfer,alquabeh2025embedding}. Our viewpoint is complementary: it turns the head-first and representation-later picture into an explicit comparison of convergence clocks.

\paragraph{Implicit regularization and low-rank bias}
A related theoretical literature studies the structural bias of gradient-based training in homogeneous, ReLU, and deep linear models \citep{soudry2018implicit,lyu2019gradient,ji2020gradient,galanti2022implicit,kou2023implicit,min2024early}. Classical and modern work on matrix approximation, matrix factorization, and spectral complexity provides the low-rank language used here \citep{eckart1936approximation,mirsky1960symmetric,candes2009exact,recht2010guaranteed,gunasekar2018implicit,arora2019implicit}. Other work connects weight decay or related regularizers to low-rank structure in linear and nonlinear networks \citep{huh2021low,timor2023implicit,abbasi2025lowrank,bartlett2017spectrally}. We use these tools at a narrower resolution: the goal is to compare how long the classifier and representation components take to reach their respective targets, with low rank used as one tractable structural marker.

\paragraph{Geometry of representations}
A parallel line of work studies representation geometry through neural collapse, class-separation laws, topology, and other structural summaries \citep{papyan2020prevalence,zhu2021geometric,mixon2020neuralcollapse,naitzat2020topology,rieck2019neural,ballester2023topology}. These perspectives are close in spirit to the representation clock, but their main object is usually the geometry attained late in training. The present paper isolates the time needed to approach such a geometry and compares it with the time needed to fit the labels.

\subsection{Assumptions, Limitations, and Roadmap}
\label{subsec:scope}
The paper has two layers of claims. The rate theorems for the two clocks are proved for the deep-linear surrogate under stated late-time assumptions. The nonlinear conclusions are conditional transfer statements for ReLU networks after gate stabilization. We do not prove from first principles that arbitrary ReLU networks must stabilize their gates, nor do we claim that low rank is the only mechanism behind grokking.

Kernel and random-feature viewpoints remain useful for the early phase of training \citep{rahimi2008rf,jacot2018ntk,lee2019wide,chizat2019lazy}, but by construction they suppress late feature motion. Deep linear networks are the cleanest setting in which the loss tail, singular-value dynamics, and weight-decay-induced spectral penalty can be studied together \citep{BaldiHornik1989,saxe2013exact,Kawaguchi2016,LaurentVonBrecht2018,gunasekar2018implicit,arora2019implicit,shang2020unified}. The ReLU bridge then states precisely when this linear analysis can be read as a local model for a nonlinear network.

Section~\ref{sec:preliminaries} defines the two clocks and the spectral quantities used later. Section~\ref{sec:time-scale-separation} proves the fast classifier clock. Section~\ref{sec:wd-low-rank-generalization} develops the slow representation clock and the conditional ReLU bridge. Section~\ref{sec:low-rank-leads-to-generalization} connects low effective dimension to generalization bounds and states the final temporal-mismatch theorem. The appendix begins with a proof index, and each theorem, lemma, proposition, and corollary in the main text is followed by either a proof idea with an appendix pointer or a direct proof pointer.

\section{Preliminaries}
\label{sec:preliminaries}
This section fixes the objects that will be used to compare the two clocks. We first define the clocks themselves, then recall stable rank, the modular-addition geometry used in the figures, the Schatten penalty induced by layerwise weight decay in deep linear networks, and the robustness framework used later to convert low effective dimension into a generalization statement.

\subsection{Two Training Clocks}
\label{subsec:two-clocks}
We need a formal object that separates ``the classifier has fitted'' from ``the representation has simplified.'' Let $A_t$ denote the effective map at training time $t$, and let $\mathcal{L}_{\mathrm{CE}}(A_t)$ be the empirical cross-entropy loss. The classifier clock records the first time at which the loss tail is below a prescribed level:
\begin{equation}\label{eq:classifier-clock}
T_{\mathrm{cls}}(\varepsilon)
:=
\inf\{t:\mathcal{L}_{\mathrm{CE}}(A_t)\le \varepsilon\}.
\end{equation}
This clock measures the fitting stage. A second clock is needed to track when the internal representation becomes simple.

To measure the second motion, let $\mathcal{S}(A_t)$ be a nonnegative structural gap, such as a Schatten-penalty gap $\Psi_p(A_t)-\Psi_p(A_\star)$, a trailing singular-value tail, or a stable-rank proxy. The representation clock is
\begin{equation}\label{eq:representation-clock}
T_{\mathrm{rep}}(\eta)
:=
\inf\{t:\mathcal{S}(A_t)\le \eta\}.
\end{equation}
The central claim of the paper is that grokking can occur when $T_{\mathrm{cls}}(\varepsilon)$ is logarithmic in the target accuracy while $T_{\mathrm{rep}}(\eta)$ is polynomial in the target structural precision.

\subsection{Stable Rank and Effective Dimension}
Let $A\in\mathbb{R}^{m\times n}$ have singular values $s_1(A)\ge s_2(A)\ge\cdots$. The stable rank is the scale-invariant quantity
\begin{equation}\label{eq:stable-rank}
\SRank(A)
:=
\frac{\|A\|_{\mathrm F}^2}{\|A\|_2^2}
=
\frac{\displaystyle\sum_{i\ge 1} s_i(A)^2}{s_1(A)^2}.
\end{equation}
It always satisfies $\SRank(A)\le \Rank(A)$, but unlike rank it is insensitive to a long tail of very small singular values \citep{recht2010guaranteed}. In this paper, low stable rank should be read as an \emph{output-effective} dimensionality statement: the map can be well approximated by a low-dimensional principal subspace. The statement concerns the output-effective map, so individual embedding vectors may still have components outside that subspace.

\subsection{Modular Addition and Cyclic Geometry}
The modular-addition task maps a pair of residues to their sum,
\begin{equation}\label{eq:mod-addition}
f(a,b)=(a+b)\bmod n,
\qquad a,b\in\{0,1,\ldots,n-1\}.
\end{equation}
The target structure is the cyclic group $(\mathbb{Z}_n,+)$. Its characters can be represented by points on the unit circle,
\begin{equation}\label{eq:cyclic-embedding}
k\longmapsto e^{2\pi i k/n},
\qquad k=0,1,\ldots,n-1.
\end{equation}
This two-dimensional cyclic geometry is the target structure the network must discover during training.

\subsection{Schatten Regularization from Layerwise Weight Decay}
The slow clock is driven by a spectral penalty. For a matrix $A$ with singular values $s_i(A)$, the Schatten-$p$ quasi-norm is
\begin{equation}\label{eq:schatten-norm}
\|A\|_{S_p}
=
\left(\sum_{i\ge 1} s_i(A)^p\right)^{\frac{1}{p}}.
\end{equation}
For $p<2$, this penalty favors concentrated singular spectra; for $p<1$ it is nonconvex and must be handled on smooth spectral strata or through limiting-subgradient language.

For a depth-$L$ linear factorization $A=W_L\cdots W_1$, layerwise Frobenius weight decay induces the following end-to-end penalty \citep{shang2020unified}:
\begin{equation}\label{eq:schatten-equivalence}
\|A\|_{S_{2/L}}^{2/L}
=
\frac{1}{L}\min_{W_L\cdots W_1=A}\sum_{\ell=1}^L\|W_\ell\|_{\mathrm F}^2.
\end{equation}
Thus, in the deep-linear surrogate, ordinary layerwise weight decay becomes a Schatten-type penalty on the product map. This is the mathematical source of the representation clock studied in Section~\ref{sec:wd-low-rank-generalization}.

\subsection{Robustness and Generalization}
\label{sec:robust-and-generalization}
The generalization argument later uses the robustness interface of \citet{xu2012robustness}. If a learned representation is effectively low-dimensional, it can be covered by fewer cells; if the loss is stable inside those cells, the empirical loss and population loss are close.

\begin{definition}[Robustness]
\label{def:robustness}
Let $\mathcal{A}$ be a learning algorithm. We say that $\mathcal{A}$ is $(K,\epsilon(\cdot))$-robust if, for every sample set $s\in\mathcal{Z}^n$, there exists a partition $\{C_i\}_{i=1}^{K}$ of $\mathcal{Z}$ such that for any $z,z'\in C_i$,
\[
    \bigl|l(\mathcal{A}_s,z)-l(\mathcal{A}_s,z')\bigr|\le \epsilon(s).
\]
\end{definition}
The definition separates a geometric quantity, the number of cells $K$, from an analytic quantity, the within-cell loss variation $\epsilon(s)$.

\begin{theorem}[Robust Generalization Bound]
\label{thm:robust-and-generalization}
Assume that the loss is bounded by $M$ and that $\mathcal{A}$ is $(K,\epsilon(\cdot))$-robust. Let $s$ be a training set of $n$ i.i.d. samples. Then, for any $\delta>0$, with probability at least $1-\delta$,
\[
    \bigl|\mathcal{L}(\mathcal{A}_s)-l_{\mathrm{emp}}(\mathcal{A}_s)\bigr|
    \le
    \epsilon(s)+M\sqrt{\frac{2K\ln 2+2\ln(1/\delta)}{n}}.
\]
\end{theorem}
\proofat{app:proof:robust-generalization}
The bound will be applied after projecting the representation onto the principal subspace selected by the output map. The stable-rank analysis controls the covering number of that subspace; robustness controls the loss variation inside each cell.

Our empirical running example is modular addition with $p^2$ examples and $p$ classes. The figures use a finite-width ReLU MLP with embeddings, cross-entropy loss, and weight decay. The theorem-level statements are proved either for the deep-linear surrogate or for conditional ReLU reductions after activation patterns have stabilized.

\section{Fast Classifier Clock: Cross-Entropy Tail After Margin Growth}
\label{sec:time-scale-separation}
The first clock asks how quickly the effective classifier can drive the empirical cross-entropy below a target level.  We therefore state the fast-clock result under an explicit post-margin gap-growth or tail-contraction condition.

\subsection{Setup and Logit Gaps}
Consider a $K$-class training set $\{(x_a,y_a)\}_{a=1}^M$, where $x_a\in\mathbb{R}^d$, $y_a\in\{1,\ldots,K\}$, and $\|x_a\|_2\le C$. The effective linear classifier is $A\in\mathbb{R}^{K\times d}$ with rows $w_1^\top,\ldots,w_K^\top$. For sample $a$, the incorrect-class gap against class $m\ne y_a$ is
\begin{equation}\label{eq:gap-definition}
\Gamma_{a,m}(A)
:=
(w_{y_a}-w_m)^\top x_a .
\end{equation}
A large positive gap means that the correct logit exceeds the corresponding incorrect logit.

The sample cross-entropy can be written entirely in terms of these gaps:
\begin{equation}\label{eq:ce-gap-form}
\ell_a(A)
=
\log\left(1+\sum_{m\ne y_a}e^{-\Gamma_{a,m}(A)}\right),
\qquad
\mathcal{L}_{\mathrm{CE}}(A)=\frac{1}{M}\sum_{a=1}^M \ell_a(A).
\end{equation}
This identity is the reason margins are the natural variable for the classifier clock.

\subsection{The Tail Bound}
Once all incorrect-class gaps are large, the loss is controlled by the exponential tail of the softmax.

\begin{lemma}[Cross-entropy tail under a uniform gap]\label{lem:exp-decay}
Suppose that for some $\gamma>0$,
\begin{equation}\label{eq:margin}
\Gamma_{a,m}(A)\ge \gamma
\qquad\text{for every }a\text{ and every }m\ne y_a.
\end{equation}
Then
\[
\ell_a(A)
\le
\log\bigl(1+(K-1)e^{-\gamma}\bigr)
\le
(K-1)e^{-\gamma},
\]
and therefore $\mathcal{L}_{\mathrm{CE}}(A)\le (K-1)e^{-\gamma}$.
\end{lemma}

\begin{proofidea}
Equation~\eqref{eq:ce-gap-form} expresses the loss as the logarithm of one plus a sum of incorrect-class exponentials. Under~\eqref{eq:margin}, every exponential term is at most $e^{-\gamma}$, and the inequality $\log(1+u)\le u$ gives the claim. The detailed proof is in Appendix~\ref{appendix:proof:lem:exp-decay}.
\end{proofidea}

For later reference, we also record the smoothness property of the loss. The constant below is used only for one-step Taylor estimates, leaving the PL issue separate.

\begin{lemma}[Global smoothness of the softmax loss]\label{lem:lipschitz}
Assume $\|x_a\|_2\le C$ for all samples. Then $\mathcal{L}_{\mathrm{CE}}$ is $L$-smooth in $A$ with
\[
L\le \frac{C^2}{2}.
\]
Equivalently, for all $A,B\in\mathbb{R}^{K\times d}$,
\[
\mathcal{L}_{\mathrm{CE}}(B)
\le
\mathcal{L}_{\mathrm{CE}}(A)
+
\langle\nabla\mathcal{L}_{\mathrm{CE}}(A),B-A\rangle
+
\frac{L}{2}\|B-A\|_{\mathrm F}^2.
\]
\end{lemma}
\proofat{appendix:proof:lem:lipschitz}

The next observation is included to make the assumptions of the fast-clock theorem explicit. A margin condition controls the value of the loss; geometric descent under gradient descent requires an additional dynamical condition.

\begin{lemma}[A fixed margin is not a uniform PL condition]\label{lem:gradient-lb}
There is no constant $\beta>0$ such that the implication
\[
\Gamma_{a,m}(A)\ge \gamma\ \text{ for all }a,m\ne y_a
\quad\Longrightarrow\quad
\|\nabla\mathcal{L}_{\mathrm{CE}}(A)\|_{\mathrm F}^2\ge \beta\mathcal{L}_{\mathrm{CE}}(A)
\]
holds uniformly over all larger margins. In particular, margin persistence alone is insufficient to justify geometric decay of the cross-entropy loss.
\end{lemma}

\begin{proofidea}
In the binary one-sample case, write the margin as $z$ and the loss as $\ell(z)=\log(1+e^{-z})$. Then $\ell(z)\asymp e^{-z}$ but $|\ell'(z)|^2\asymp e^{-2z}$ as $z\to\infty$, so $|\ell'(z)|^2/\ell(z)\to0$. The full proof is in Appendix~\ref{appendix:proof:lem:gradient-lb}.
\end{proofidea}

\subsection{A Sufficient Condition for the Fast Clock}
The correct fast-clock statement is therefore conditional on a post-margin mechanism that keeps increasing the relevant logit gaps. The following assumption is deliberately phrased at the level of the effective classifier; Section~\ref{sec:relu-reduction} explains how a head-dominated ReLU phase can supply such a mechanism locally.

\begin{assumption}[Post-margin gap growth]\label{assump:margin-persist}
There exist constants $t_0\ge0$, $\gamma_0\ge0$, and $\nu>0$ such that along the classifier trajectory,
\begin{equation}\label{eq:gap-growth}
\Gamma_{a,m}(A_t)
\ge
\gamma_0+\nu(t-t_0)
\qquad\text{for every }t\ge t_0,
\end{equation}
and for every sample $a$ and incorrect class $m\ne y_a$.
\end{assumption}
The assumption says that after the classifier has entered the margin regime, the separating gaps continue to grow at a definite rate. It is a sufficient condition for the fast clock.

\begin{theorem}[Fast classifier clock after gap growth]\label{thm:main-exp-convergence}
Under Assumption~\ref{assump:margin-persist}, the empirical cross-entropy satisfies
\[
\mathcal{L}_{\mathrm{CE}}(A_t)
\le
(K-1)\exp\{-\gamma_0-\nu(t-t_0)\}
\qquad(t\ge t_0).
\]
Consequently, the classifier clock defined in~\eqref{eq:classifier-clock} obeys
\[
T_{\mathrm{cls}}(\varepsilon)
\le
 t_0+\frac{1}{\nu}\log\!\left(\frac{(K-1)e^{-\gamma_0}}{\varepsilon}\right)
=
O\!\left(\log(1/\varepsilon)\right).
\]
\end{theorem}

\paragraph{Interpretation}
This theorem is the fast side of the two-clock picture. Once the classifier gaps grow linearly in training time, the softmax tail contracts exponentially, so the number of additional steps needed to make the training loss smaller by a factor of $\varepsilon$ grows only logarithmically. The theorem isolates the condition under which the classifier clock is genuinely fast.

\begin{proofidea}
Assumption~\ref{assump:margin-persist} gives a lower bound on every incorrect-class gap. Substituting that bound into Lemma~\ref{lem:exp-decay} gives
\[
\mathcal{L}_{\mathrm{CE}}(A_t)
\le (K-1)e^{-\gamma_0-\nu(t-t_0)}.
\]
Solving the inequality $(K-1)e^{-\gamma_0-\nu(t-t_0)}\le\varepsilon$ gives the displayed clock bound. See Appendix~\ref{appendix:proof:thm:main-exp-convergence} for the complete proof.
\end{proofidea}

The same conclusion can be stated directly as a one-step contraction condition, which is sometimes easier to verify for a discrete optimizer.

\begin{proposition}[Tail contraction implies the same fast clock]\label{thm:single-step-decay}
Suppose there are $t_0$ and $\mu\in(0,1)$ such that for all $t\ge t_0$,
\[
\mathcal{L}_{\mathrm{CE}}(A_{t+1})\le (1-\mu)\mathcal{L}_{\mathrm{CE}}(A_t).
\]
Then
\[
\mathcal{L}_{\mathrm{CE}}(A_t)
\le
(1-\mu)^{t-t_0}\mathcal{L}_{\mathrm{CE}}(A_{t_0}),
\qquad
T_{\mathrm{cls}}(\varepsilon)
=
O\!\left(\log(1/\varepsilon)\right).
\]
\end{proposition}
\proofat{app:proof:prop:tail-contraction}

\paragraph{Connection to grokking}
The fast-clock theorem explains why the training loss can become visually flat early in a grokking run. It also clarifies what remains to be explained: a small cross-entropy tail only says that the classifier separates the training samples. It leaves open whether the representation has low effective dimension, cyclic geometry, or any other rule-aligned structure. Those properties belong to the representation clock developed next.

\section{Slow Representation Clock: Spectral Simplification Under Weight Decay}
\label{sec:wd-low-rank-generalization}
The second clock measures a different object from the classifier loss. After the classifier has fitted the training set, weight decay can still reshape the end-to-end map by shrinking weak singular directions and concentrating the spectrum. This section develops the representation clock in the deep-linear surrogate and then states the conditional ReLU bridge that connects the surrogate to the finite-width experiments.

\subsection{Empirical Illustration of Low-Rank Structure}
Modular addition provides a useful visual case because the target rule has a simple cyclic representation.

\begin{figure}[t!]
\centering
\includegraphics[width=0.9\linewidth]{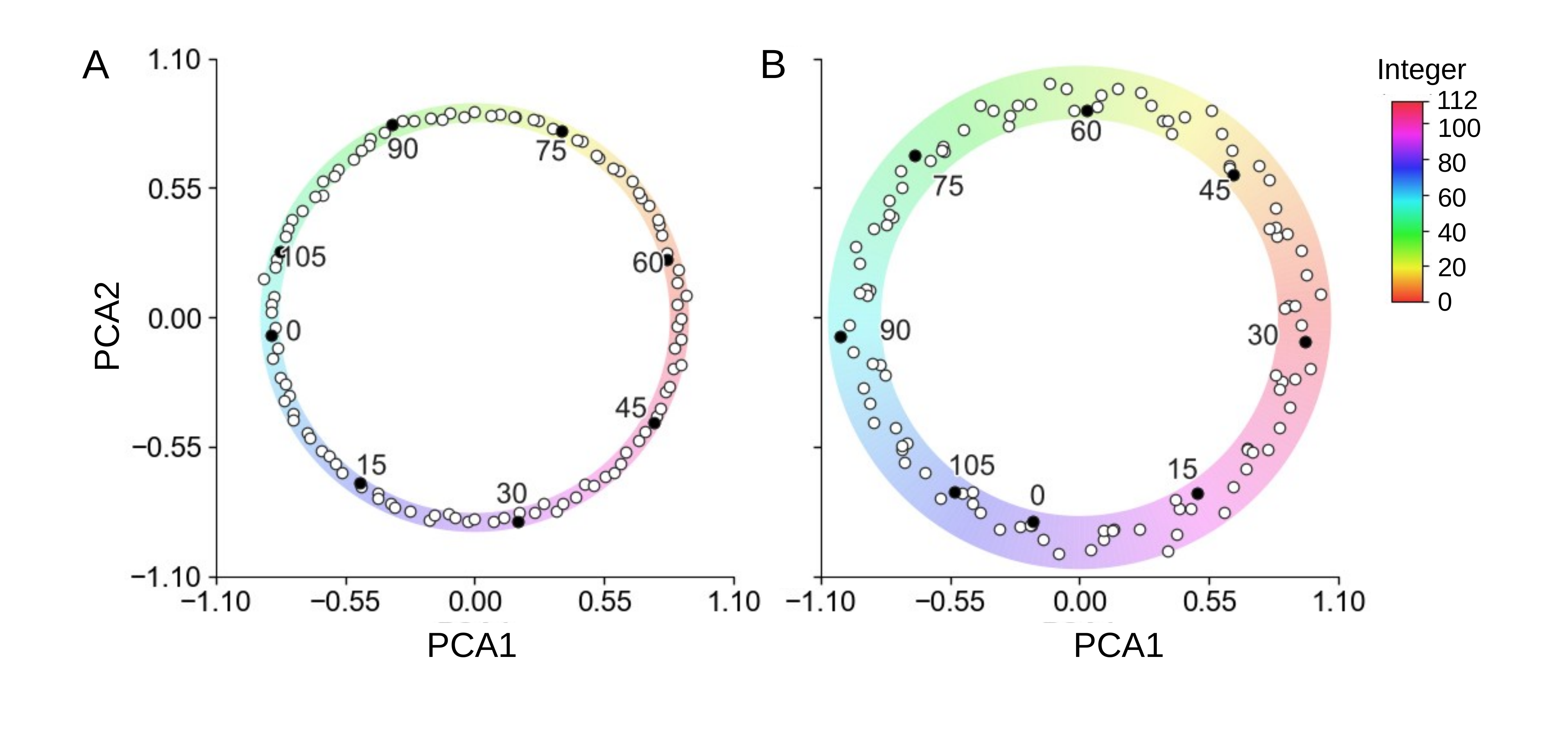}
\caption{Learned modular-addition geometry for a ReLU MLP trained on addition mod $113$. (A) Token embeddings after generator-based label reordering. (B) Output weights after the same label reordering. Colors indicate labels. The display is qualitative evidence for the cyclic organization tracked by the empirical representation clock; theorem-level guarantees are stated for the deep-linear surrogate and the conditional ReLU transfer statements.}
\label{fig:embedding-weight}
\end{figure}

\begin{figure}[htb]
\centering
\includegraphics[width=0.7\linewidth]{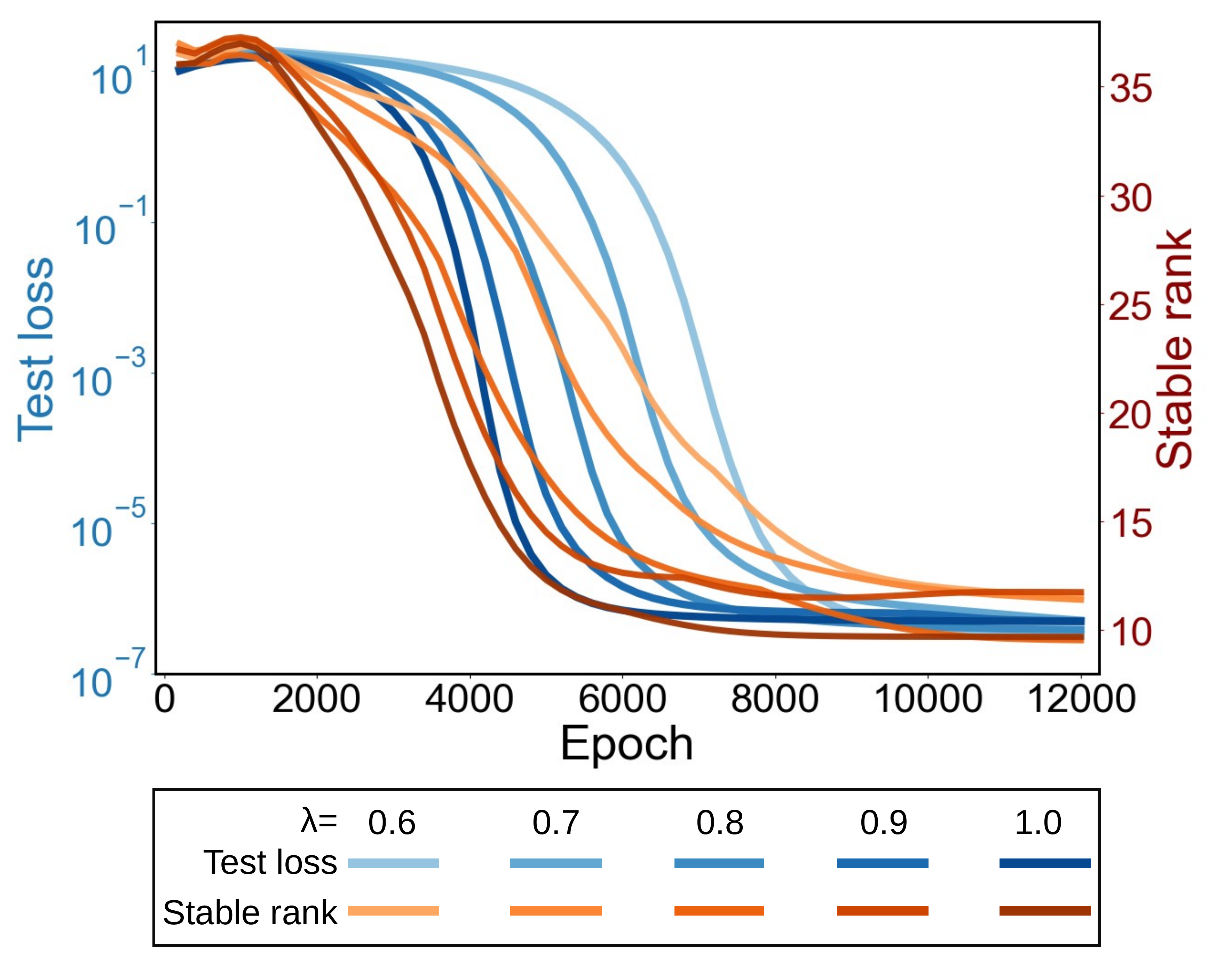}
\caption{Test loss and stable rank during training for a ReLU MLP on modular addition with weight decay $\lambda\in\{0.6,0.7,0.8,0.9,1.0\}$. Test loss is plotted on the left $y$-axis with log scale, and stable rank is plotted on the right $y$-axis. The late decrease in stable rank is the empirical representation-clock diagnostic used as qualitative support for the theory.}
\label{fig:test-loss-low-rank}
\end{figure}

\subsection{A Conditional Reduction Principle for ReLU Networks}
\label{sec:relu-reduction}
The deep-linear analysis is useful for ReLU networks only after one states how nonlinearity is being controlled. We use two exact facts. The first is geometric: fixed activation masks make the network affine on the training samples. The second is differential: in a two-layer embedding model, lower-layer gradients contain extra downstream operator-norm factors compared with the head gradient.

\paragraph{Frozen-gate reduction}
Consider an $L$-layer ReLU network
\[
f_\Theta(x)=W_{L+1}\phi\bigl(W_L\phi(\cdots \phi(W_1x+b_1)\cdots)+b_L\bigr)+b_{L+1},
\qquad \phi(u)=\max\{u,0\}.
\]
For a training sample $x_i$, write $D_{\ell,i}(t)=\Diag(\mathbf 1_{a_{\ell,i}(t)>0})$ for the diagonal activation mask at layer $\ell$.

\begin{proposition}[Frozen active subsystem]\label{prop:relu-frozen-gate}
Fix a finite training set $\mathcal D=\{x_i\}_{i=1}^n$. Suppose there is an interval $I=[T_0,T_1]$ such that for every sample $x_i$ and layer $\ell$,
\[
D_{\ell,i}(t)=D_{\ell,i}^{\star},
\qquad t\in I.
\]
Then for every $i$ and every $t\in I$,
\[
f_\Theta(x_i)=A_i(t)x_i+c_i(t),
\qquad
A_i(t)=W_{L+1}(t)D^{\star}_{L,i}W_L(t)\cdots D^{\star}_{1,i}W_1(t),
\]
for an affine offset $c_i(t)$ determined by the frozen masks and biases. If all layers below the classifier head are also fixed on $I$, then training the head is exactly multiclass logistic regression on the fixed active features.
\end{proposition}
\proofat{app:proof:prop:relu-frozen-gate}

\paragraph{Head-dominated gradient hierarchy}
For the modular-addition architecture with a trainable embedding table, write
\[
f_\Theta(x)=U\phi(Wz_x+b),
\qquad z_x=B_xE,
\]
where $E$ is the embedding matrix and $B_x$ selects the embeddings used by the input $x$.

\begin{proposition}[Head gradients and embedding gradients]\label{prop:head-dominant-gradient}
For one sample $(x,y)$ with cross-entropy loss $\ell_x$, let $p_x=\softmax(f_\Theta(x))$, $\delta_x=p_x-e_y$, $a_x=Wz_x+b$, $D_x=\Diag(\mathbf 1_{a_x>0})$, and $h_x=\phi(a_x)$. Then
\[
\nabla_U \ell_x = \delta_x h_x^{\top},
\qquad
\nabla_W \ell_x = D_xU^{\top}\delta_x z_x^{\top},
\qquad
\nabla_E \ell_x = B_x^{\top}W^{\top}D_xU^{\top}\delta_x.
\]
Consequently,
\[
\|\nabla_W\ell_x\|_{\mathrm F}
\le \|U\|_2\|z_x\|_2\|\delta_x\|_2,
\qquad
\|\nabla_E\ell_x\|_{\mathrm F}
\le \|B_x\|_2\|W\|_2\|U\|_2\|\delta_x\|_2,
\]
whereas
\[
\|\nabla_U\ell_x\|_{\mathrm F}=\|\delta_x\|_2\|h_x\|_2.
\]
Thus, when $\|U\|_2$ and $\|W\|_2$ are small or controlled and $\|h_x\|_2$ is not negligible, the embedding and lower-layer gradients carry additional multiplicative factors relative to the head gradient.
\end{proposition}
\proofat{app:proof:prop:head-dominant-gradient}

\paragraph{Interpretation}
The propositions identify the stabilized regimes in which a nonlinear network can be compared to an active linear subsystem, and in which the classifier head can move before the embedding block. This is the ReLU counterpart of the two-clock narrative: the classifier clock can finish before the representation clock has reached its low-dimensional structure.

\subsection{Deep Linear Structural Model}
\label{sec:math-setup}
The representation clock is analyzed in a deep linear network with end-to-end map
\begin{equation}\label{eq:end-to-end-map}
A=W_L\cdots W_1\in\mathbb{R}^{k\times d}.
\end{equation}
Layerwise weight decay induces the Schatten penalty in~\eqref{eq:schatten-equivalence}. We therefore study the regularized end-to-end energy
\begin{equation}\label{eq:regularized-energy}
\mathcal{E}(A)=\mathcal{L}_{\mathrm{CE}}(A)+\Psi_p(A).
\end{equation}
The spectral part of this energy is the Schatten-type penalty
\[
\Psi_p(A)=\lambda\|A\|_{S_p}^{p},
\qquad 0<p<1.
\]
The penalty is nonconvex and singular at zero singular values. For this reason, all differential statements are made either on smooth spectral strata or for limiting-subgradient trajectories with an explicit late-time stratum assumption \citep{filippov1988differential,rockafellar1998variational}.

\subsection{Spectral Shrinkage on a Smooth Stratum}
\label{sec:singular-value-dynamics}
The basic singular-value mechanism is local and exact. Suppose $A(t)$ remains on a smooth spectral stratum where the singular value $s_i(t)$ is positive and simple. Write $A(t)=U(t)\Sigma(t)V(t)^\top$ and let $u_i(t),v_i(t)$ be the corresponding singular vectors. Along the smooth gradient flow of~\eqref{eq:regularized-energy}, the singular value derivative is
\begin{equation}\label{eq:singular-derivative}
\dot s_i(t)
=
-\big\langle \nabla\mathcal{L}_{\mathrm{CE}}(A(t)),u_i(t)v_i(t)^\top\big\rangle
-\lambda p\,s_i(t)^{p-1}.
\end{equation}
The first term is the data force in the $i$-th singular direction; the second term is the spectral shrinkage created by weight decay.

\begin{proposition}[Projected singular-value drift]\label{thm:gradient-projection}
On any time interval where $A(t)$ stays on a smooth spectral stratum and $s_i(t)>0$ is simple, define
\[
g_i(t):=-\big\langle \nabla\mathcal{L}_{\mathrm{CE}}(A(t)),u_i(t)v_i(t)^\top\big\rangle.
\]
Then, for almost every $t$ in that interval,
\[
\dot s_i(t)=g_i(t)-\lambda p\,s_i(t)^{p-1}.
\]
In particular, for $0<p<1$, smaller positive singular values experience stronger shrinkage from the regularizer.
\end{proposition}
\proofat{app:proof:prop:gradient-projection}

\paragraph{Interpretation}
The formula is best read as a selection mechanism. The term $s_i^{p-1}$ diverges as $s_i\downarrow0$, so weak singular directions are penalized aggressively. At singular collisions or at zero singular values, one must use a stratum change or a limiting-subgradient description; the paper uses~\eqref{eq:singular-derivative} as a local drift identity, with the global rate handled by the KL-tail assumptions below.

\begin{theorem}[Stable-rank consequence of spectral selection]\label{thm:stable-rank}
Suppose the trajectory converges to a nonzero matrix $A_\star$ with $r_+\ge 1$ nonzero singular values. Then
\[
1\le \SRank(A_\star)\le r_+.
\]
Consequently, any mechanism that drives many singular values to zero or to a negligible tail yields a low-effective-dimensional end-to-end map.
\end{theorem}
\proofat{app:proof:thm:stable-rank}

\paragraph{Interpretation}
The inequality is a deterministic property of the limiting spectrum. The dynamical content is supplied by the shrinkage formula above and by the KL-tail analysis below, which describe why spectral simplification can continue after the classifier loss has already saturated.

\subsection{Late-Time KL Tail and the Polynomial Clock}
\label{sec:matrix-convergence}
The representation clock is a convergence statement for the energy gap
\[
\Delta(t):=\mathcal{E}(A(t))-\mathcal{E}(A_\star).
\]
A KL inequality gives the rate at which this gap can close near a limiting point \citep{kurdyka1998gradients,attouch2013convergence,bolte2014palm,rockafellar1998variational}. To avoid overclaiming an exponent from informal flatness considerations, we state the late-time geometry as an assumption.

\begin{assumption}[Late-time KL geometry]\label{ass:HOF}
There are a limit point $A_\star$, a neighborhood $\mathcal V$ of $A_\star$, constants $0<c_-\le c_+<\infty$, and an exponent $\theta\in(1/2,1)$ such that, for all late-time points in $\mathcal V$ with $\Delta(A)>0$,
\begin{equation}\label{eq:two-sided-kl}
c_-\Delta(A)^\theta
\le
\|\partial^0\mathcal{E}(A)\|_{\mathrm F}
\le
c_+\Delta(A)^\theta.
\end{equation}
Here $\partial^0\mathcal{E}(A)$ denotes the minimum-norm limiting subgradient on the active stratum. We also assume that after some time $T_{\mathrm{KL}}$, the trajectory remains in this neighborhood and in a fixed smooth spectral stratum except for finitely many stratum changes.
\end{assumption}
The lower KL bound is enough for an upper polynomial convergence rate; the matching upper bound is the nondegeneracy condition that makes the polynomial clock sharp.

\begin{lemma}[Energy dissipation]\label{lem:dissipation}
Let $A(t)$ be an absolutely continuous limiting-subgradient trajectory for~\eqref{eq:regularized-energy}. On every smooth spectral interval,
\[
\frac{d}{dt}\mathcal{E}(A(t))=-\|\nabla\mathcal{E}(A(t))\|_{\mathrm F}^2\le0.
\]
Under the finite-stratum-change condition in Assumption~\ref{ass:HOF}, the same monotone dissipation statement holds after concatenating the smooth intervals, with $\nabla\mathcal E$ replaced by $\partial^0\mathcal E$ at nonsmooth instants.
\end{lemma}
\proofat{app:proof:lem:dissipation}

\begin{definition}[KL exponent for the representation clock]
A late-time trajectory has KL exponent $\theta$ if its energy gap satisfies a bound of the form~\eqref{eq:two-sided-kl} near its limiting point. The case $\theta>1/2$ is the slow regime: integrating the resulting differential inequality produces a polynomial tail.
\end{definition}

\begin{theorem}[KL source of a slow representation clock]\label{thm:KL-global}
Under Assumption~\ref{ass:HOF}, the representation energy satisfies a sharp KL-tail relation with exponent $\theta\in(1/2,1)$ on the late-time neighborhood of $A_\star$. In particular, along the limiting-subgradient trajectory,
\[
-c_+^2\Delta(t)^{2\theta}
\le
\dot\Delta(t)
\le
-c_-^2\Delta(t)^{2\theta}
\]
for almost every sufficiently large $t$.
\end{theorem}
\proofat{app:proof:thm:KL-global}

\begin{theorem}[Polynomial representation clock]\label{thm:sharp-rate}
Assume the hypotheses of Theorem~\ref{thm:KL-global}. Then, for all $t\ge T_{\mathrm{KL}}$,
\[
\Delta(t)=\Theta\!\left((t-T_{\mathrm{KL}})^{-\frac{1}{2\theta-1}}\right).
\]
If the structural gap $\mathcal{S}(A(t))$ is comparable to $\Delta(t)$ on this late-time neighborhood, then
\[
T_{\mathrm{rep}}(\eta)-T_{\mathrm{KL}}
=
\Theta\!\left(\eta^{-(2\theta-1)}\right).
\]
With $\theta=1-1/\rho$, this exponent becomes $2\theta-1=(\rho-2)/\rho$.
\end{theorem}
\proofat{app:proof:thm:sharp-rate}

\paragraph{Interpretation}
This theorem is the slow side of the two-clock picture. Once the trajectory is in a flat KL regime with $\theta>1/2$, reducing the structural gap to precision $\eta$ takes polynomial time in $1/\eta$. The theorem is conditional because the exponent and sharpness of the KL tail are late-time geometric assumptions; this is the correct level of rigor for the structural clock.

\paragraph{Connection to grokking dynamics}
The classifier clock in Section~\ref{sec:time-scale-separation} is logarithmic once the post-margin gap-growth condition holds. The representation clock here is polynomial once the late-time KL tail controls spectral simplification. The mismatch between these two convergence laws is the mathematical point of the paper: training loss can saturate while the representation is still evolving toward a simpler geometry.

\section{From Low Effective Dimension to Delayed Generalization}
\label{sec:low-rank-leads-to-generalization}
The previous section shows why the representation clock can be slow. We now connect that structural clock to generalization. The argument is conservative: low stable rank reduces the effective dimension seen by the output map, which is the component used in the generalization bound. A smaller effective dimension improves covering-number terms in robustness-style generalization bounds.

\subsection{Geometric Complexity and Covering Numbers}
A low-dimensional Euclidean set can be covered with exponentially fewer balls than a high-dimensional one. We use the following standard comparison as the geometric input \citep{anthony1999neural,bartlett2002rademacher,ledoux1991probability,vershynin2018high}.

\begin{proposition}[Volume covering lemma]\label{prop:volume}
Let $B_2^d(R)$ denote the Euclidean ball of radius $R$ in $\mathbb{R}^d$. For $0<\alpha\le R$, its radius-$\alpha$ covering number satisfies
\[
\left(\frac{R}{\alpha}\right)^d
\lesssim
\mathcal{N}(B_2^d(R),\alpha)
\lesssim
\left(\frac{2R}{\alpha}\right)^d.
\]
Equivalently,
\[
\log \mathcal{N}(B_2^d(R),\alpha)
=
O\!\left(d\log(R/\alpha)\right).
\]
\end{proposition}
\proofat{app:proof:prop:volume}
This proposition is the only place where dimension enters the generalization bridge. The stable-rank result below replaces the ambient dimension by an output-effective dimension.

\subsection{Stable Rank Controls an Output-Effective Subspace}
Consider an embedding map $z=W_1x\in\mathbb{R}^{d_1}$ followed by a final linear map $f(z)=A'z$, where $A'\in\mathbb{R}^{k\times d_1}$. The embedding set is
\[
\mathcal{M}_{\mathrm{embed}}
:=
\{W_1x:x\in\mathcal X\}\subset\mathbb{R}^{d_1}.
\]
The embeddings may live in a high-dimensional ambient space, but the output map may depend mainly on a smaller principal subspace.

\begin{theorem}[Stable rank controls output-effective dimension]\label{thm:stable-rank-embedding}
Let $A'\in\mathbb{R}^{k\times d_1}$. For $\eta\in(0,1)$, set
\[
d_\eta=\left\lceil\frac{\SRank(A')}{\eta^2}\right\rceil.
\]
Let $A'_{d_\eta}$ be the best rank-$d_\eta$ approximation of $A'$, and let $V_{d_\eta}\subset\mathbb{R}^{d_1}$ be the span of the top $d_\eta$ right singular vectors. Then every embedding $z\in\mathcal{M}_{\mathrm{embed}}$ satisfies
\[
\|(A'-A'_{d_\eta})z\|_2
\le
\eta\|A'\|_2\|z\|_2.
\]
Thus the action of $A'$ on embeddings is approximated, up to relative output error $\eta$, by a $d_\eta$-dimensional principal subspace controlled by stable rank.
\end{theorem}

\begin{proofidea}
Eckart--Young gives $\|A'-A'_{d_\eta}\|_2=s_{d_\eta+1}(A')$. Monotonicity of singular values gives $(d_\eta+1)s_{d_\eta+1}^2\le\sum_{i\ge 1} s_i^2=\SRank(A')s_1^2$. The choice of $d_\eta$ yields $s_{d_\eta+1}\le\eta s_1$, and applying this operator-norm bound to $z$ proves the claim. The full proof is in Appendix~\ref{app:thm:stable-rank-embedding}.
\end{proofidea}

\begin{proposition}[Covering number of the principal subspace]\label{prop:embedding-covering}
Suppose $\|P_{V_{d_\eta}}z\|_2\le R$ for all $z\in\mathcal{M}_{\mathrm{embed}}$, where $P_{V_{d_\eta}}$ is the orthogonal projection onto $V_{d_\eta}$. Then
\[
\mathcal{N}(P_{V_{d_\eta}}\mathcal{M}_{\mathrm{embed}},\alpha)
\le
\left(\frac{2R}{\alpha}\right)^{d_\eta}.
\]
\end{proposition}
\proofat{app:proof:prop:embedding-covering}
The cover is for the projected representation; the full ambient embedding set may remain higher-dimensional. Theorem~\ref{thm:stable-rank-embedding} controls the output error introduced by this projection.

\subsection{A Robustness Bound with a Stable-Rank Term}
The following theorem is a direct use of the robustness bound in Theorem~\ref{thm:robust-and-generalization}, with an additional output-tail term coming from the projection error.

\begin{definition}[Projected embedding robustness]
Let $f(z)=A'z$ and let $V_{d_\eta}$ be the principal subspace from Theorem~\ref{thm:stable-rank-embedding}. We say that $f$ is $(\alpha,\epsilon)$-robust on the projected embedding if the projected set $P_{V_{d_\eta}}\mathcal{M}_{\mathrm{embed}}$ admits an $\alpha$-cover such that any two embeddings whose projections fall in the same cover cell have loss values differing by at most $\epsilon$, after accounting for the projected logits.
\end{definition}

\begin{theorem}[Generalization bound through the representation clock]\label{thm:gen-bound-stablerank}
Assume that:
\begin{enumerate}[label=(\roman*),leftmargin=2.8em]
    \item $\|P_{V_{d_\eta}}z\|_2\le R$ and $\|z\|_2\le R_z$ for all embeddings;
    \item $d_\eta=\lceil\SRank(A')/\eta^2\rceil$;
    \item the projected embedding is $(\alpha,\epsilon)$-robust;
    \item the loss is bounded by $M$ and is $G$-Lipschitz in the logits.
\end{enumerate}
Then, with probability at least $1-\delta$ over $n$ training samples,
\[
\bigl|\mathcal{L}(f)-\widehat{\mathcal L}_n(f)\bigr|
\le
\epsilon
+2G\eta\|A'\|_2R_z
+M\sqrt{\frac{2\left(\frac{2R}{\alpha}\right)^{d_\eta}\ln2+2\ln(1/\delta)}{n}}.
\]
\end{theorem}
\proofat{app:proof:thm:gen-bound-stablerank}

\paragraph{Interpretation}
As the representation clock reduces stable rank, the exponent $d_\eta$ in the covering term decreases. The theorem uses low stable rank to reduce the geometric complexity term once the projected representation is robust.

\subsection{Temporal Dynamics: Why Grokking Occurs}
\label{sec:time-mismatch-overfit-generalization}
We can now state the final comparison. The classifier clock is attached to the cross-entropy tail. The representation clock is attached to a structural gap whose reduction improves the effective dimension and, through the previous theorem, the generalization bound.

\begin{theorem}[Temporal mismatch between the two clocks]\label{thm:temporal-mismatch}
Assume the fast-clock condition of Theorem~\ref{thm:main-exp-convergence}. Assume also the sharp KL-tail condition of Theorem~\ref{thm:sharp-rate} for a structural gap $\mathcal{S}(A(t))$ comparable to the energy gap. Then
\[
T_{\mathrm{cls}}(\varepsilon)
=
O\!\left(\log(1/\varepsilon)\right),
\qquad
T_{\mathrm{rep}}(\eta)-T_{\mathrm{KL}}
=
\Theta\!\left(\eta^{-(2\theta-1)}\right),
\]
where $\theta\in(1/2,1)$ is the late-time KL exponent. If the required structural precision is tied to the target loss scale by $\eta\asymp\varepsilon^q$ for some $q>0$, then
\[
T_{\mathrm{rep}}(\varepsilon^q)-T_{\mathrm{KL}}
=
\Theta\!\left(\varepsilon^{-q(2\theta-1)}\right),
\]
which is polynomial in $1/\varepsilon$.
\end{theorem}
\proofat{app:proof:thm:temporal-mismatch}

\paragraph{Interpretation}
The theorem gives the paper's main explanation of grokking. The training loss can become small on the classifier clock, while the structural quantity needed for robust generalization continues to evolve on the representation clock. Delayed test accuracy is therefore compatible with rapid interpolation: the first clock has finished, but the second has not.

\begin{corollary}[Conditional ReLU version]\label{cor:relu-temporal-mismatch}
Assume a ReLU network enters a time window in which: (i) the classifier head has a post-margin gap-growth or tail-contraction regime on the active features; (ii) activation patterns on the training set are stable; and (iii) the late active subsystem satisfies the sharp KL-tail condition for its structural gap. Then the stabilized subsystem inherits the same logarithmic classifier clock and polynomial representation clock.
\end{corollary}
\proofat{app:proof:cor:relu-temporal-mismatch}

\paragraph{From the clocks to the observed plateau}
The theorem suggests the following reading of the grokking curve. First, the easiest degrees of freedom, often the classifier head or near-head directions, fit the labels and drive down training loss. Second, the spectral penalty keeps reshaping the effective map even though the empirical loss is already small. Third, once the representation has simplified enough to reduce the effective dimension and align with the task geometry, the generalization bound and the observed test accuracy can improve. The plateau is the visible interval between the saturation of the first clock and the completion of the second.

\section{Conclusion and Discussion}
This paper reorganizes grokking around a single dynamical object: two training clocks. The classifier clock measures how quickly the effective classifier drives down the cross-entropy tail. The representation clock measures how quickly the learned map simplifies under spectral regularization. In the deep-linear core model, the first clock is logarithmic under an explicit post-margin gap-growth or tail-contraction condition, while the second clock is polynomial under a sharp late-time KL-tail condition. This convergence-time difference gives a concrete explanation of why a model can interpolate the training set long before its test accuracy improves.

Margin controls the value of the cross-entropy tail, but the softmax Hessian degenerates as margins grow, so the fast clock is stated under a gap-growth or tail-contraction condition. Similarly, the slow clock is stated under a late-time KL geometry assumption; the KL exponent is treated as part of the hypothesis that determines the polynomial representation-clock rate.

The ReLU results are conditional but precise. If activation patterns stabilize on the training set, the network reduces to an active affine subsystem on those samples. If the classifier head is in a head-dominated regime, it can fit before the embedding block and the active representation have finished moving. These statements do not prove global ReLU grokking from first principles, but they explain when the deep-linear two-clock analysis becomes a faithful local model of a nonlinear training run.

Several limitations remain. Arbitrary ReLU networks may fail to enter a stabilized gate regime, and low rank is one simplicity bias among several possible routes behind grokking. We also separate the clock comparison from the empirical question of which structural metric best predicts test accuracy in every task. The paper's narrower claim is that, whenever fitting and simplification are governed by different convergence laws, delayed generalization is not paradoxical: the classifier has finished its job before the representation has finished its own.

Future work should try to derive the post-margin gap-growth condition and the late-time KL exponent from primitive assumptions on finite-width nonlinear networks. It would also be useful to compare the spectral representation clock studied here with other slow clocks, such as sparsity, invariance formation, neural-collapse geometry, or circuit formation in transformers. Such comparisons may reveal whether grokking is a single mechanism or a family of delayed-simplification phenomena.

\section*{Acknowledgments}
\begingroup\small
This work was supported by the CAS Project for Young Scientists in Basic Research (No.~YSBR-034 to S.Z.), the Strategic Priority Research Program of the Chinese Academy of Sciences (No.~XDB0680101 to S.Z.), the National Natural Science Foundation of China (Nos.~32341013 and 12326614), and the Robotic AI-Scientist Platform of the Chinese Academy of Sciences. The authors declare no competing interests.
\endgroup

\bibliographystyle{plainnat}
\bibliography{sample}
\clearpage
\appendix
\setcounter{table}{0}
\renewcommand{\thetable}{S\arabic{table}}
\renewcommand{\theHtable}{S\arabic{table}}

\section{Supplementary Summary Tables}
\label{app:supplementary-summary-tables}

\begin{table}[htbp]
\centering
\small
\setlength{\belowcaptionskip}{6pt}
\caption{Representative grokking settings. Each row records the task family, model class, data or setup, explicit regularization axis used in the experiment, and source.}
\label{tab:grokking}
\setlength{\tabcolsep}{3.8pt}
\renewcommand{\arraystretch}{1.08}
\resizebox{\linewidth}{!}{%
\begin{tabular}{@{}lllll@{}}
\toprule
\rowcolor{gray!15}
\textbf{Task} & \textbf{Model} & \textbf{Data / setup} & \textbf{Reg. axis} & \textbf{Ref.} \\
\midrule
Modular div.        & Transformer      & mod 97, 50\% train      & $\WD=1$        & \citealp{power2022grokking} \\
Modular add.        & Transformer      & mod 97 / 113            & $\WD$ / init.  & \citealp{power2022grokking} \\
Modular arith.      & 2-layer MLP       & mod 97, MSE, full GD    & none explicit  & \citealp{gromov2023grokking} \\
Group operations    & Transformer      & finite groups, e.g. $S_5$ & $\WD=1$      & \citealp{power2022grokking} \\
Local-rule models   & Perceptron / TN   & solvable rules, rule-30 & $\ell_1/\ell_2$ & \citealp{JMLR:v25:22-1228} \\
Modular polynomials & MLP / analytic net & mult. / multi-add.     & task dependent & \citealp{doshi2024grokkingpolynomials} \\
Finite alg. data    & Transformer      & 5-digit add., repeats   & $\WD=1$        & \citealp{nanda2023progress} \\
MNIST               & 3-layer ReLU MLP & 1k images, large init.  & $\WD$ sweep    & \citealp{liu2022omnigrok} \\
IMDb                & 2-layer LSTM     & 1k reviews, large init. & $\WD=1$        & \citealp{liu2022omnigrok} \\
QM9 molecules       & GCNN             & 200 samples, large init. & none explicit & \citealp{liu2022omnigrok} \\
Ungrokking          & Transformer      & mod add., reduced data  & $\WD$ sweep    & \citealp{varma2023explaining} \\
Information probes  & 2-layer FCN      & mod add., mod 97        & $\WD$ / init.  & \citealp{clauw2024information} \\
\bottomrule
\end{tabular}%
}
\end{table}

\begin{table}[htbp]
\centering
\small
\setlength{\belowcaptionskip}{6pt}
\caption{Scope of the main two-clock claims. The table separates theorem-level deep-linear statements, exact conditional ReLU statements, and empirical support, and records the assumptions needed for each claim.}
\label{tab:claim-scope}
\begin{threeparttable}
\setlength{\tabcolsep}{6pt}
\renewcommand{\arraystretch}{1.16}

\begin{tabularx}{\linewidth}{@{}L{3.55cm}Y@{}}
\toprule
\textbf{Claim} & \textbf{Scope and assumptions} \\
\midrule

Fast classifier clock &
\emph{Conditional theorem for deep linear networks and frozen-gate ReLU heads.}
The result assumes post-margin gap growth or one-step tail contraction, together with bounded features.  \\

\addlinespace[2pt]

Polynomial representation clock &
\emph{Conditional theorem for deep linear networks.}
The result uses layerwise weight decay through its Schatten-$p$ surrogate, a late-time sharp KL geometry with exponent $\theta>1/2$, and only finitely many changes of the smooth spectral stratum. \\

\addlinespace[2pt]

Head-first, representation-later hierarchy &
\emph{Exact gradient identity for a two-layer ReLU embedding model.}
The identity follows from chain-rule gradient bounds. The head-dominance interpretation additionally requires bounded downstream factors and nonvanishing hidden features. \\

\addlinespace[2pt]

Conditional delayed generalization &
\emph{Empirical evidence plus conditional transfer for ReLU MLPs.}
The transfer requires stabilized activation patterns on the training set and continued low-effective-dimension drift of the active linear map. \\

\bottomrule
\end{tabularx}

\begin{tablenotes}[flushleft]
\footnotesize
\item \emph{Note.}
The table distinguishes theorem-level statements, conditional ReLU transfers, and empirical evidence. It records the scope of each claim and introduces no additional assumptions.
\end{tablenotes}

\end{threeparttable}
\end{table}

\FloatBarrier

\section{Experimental Settings}\label{app:experimental-setting}

In this appendix, we summarize the experimental setup for the modular-addition ReLU MLP used throughout the figures. The purpose of these experiments is illustrative: they show the qualitative delayed-generalization and low-rank trends that motivate the theory, but the theorems in the main text concern deep linear surrogates or conditional frozen-gate ReLU reductions.

\subsection{Dataset}
We study modular addition over a prime modulus $p$ (with $p=113$ in the main figures unless stated otherwise). Each example is an ordered pair $(a,b)\in\{0,1,\dots,p-1\}^2$ with label
\[
y=(a+b)\bmod p.
\]
This gives $p^2$ total examples spread across $p$ classes. The data set is split into training and test subsets using the same train/test partition across all weight-decay settings so that only the optimization dynamics vary.

\subsection{Model Architecture}
The empirical model is a symmetric MLP with trainable token embeddings:
\begin{itemize}
    \item \textbf{Token embeddings.} Each integer in $\{0,1,\dots,p-1\}$ is mapped to a trainable embedding vector in $\mathbb{R}^{128}$.
    \item \textbf{Symmetric combination.} Given an input pair $(a,b)$, the corresponding embeddings are combined symmetrically by elementwise addition. This enforces the commutativity of modular addition at the input-representation level.
    \item \textbf{Nonlinear classifier.} The combined vector is passed through fully connected layers with ReLU nonlinearities, followed by a final linear classifier and softmax / cross-entropy loss.
\end{itemize}
The model is intentionally small and finite-width so that it operates outside the clean NTK limit and visibly exhibits feature learning.

\subsection{Training Procedure}
We use full-batch training and run optimization long enough to capture the delayed generalization regime. The main sweeps vary the weight-decay coefficient over
\[
\lambda\in\{0.6,0.7,0.8,0.9,1.0\},
\]
while keeping the remaining hyperparameters fixed. Models are evaluated periodically during training to record training loss, test loss, accuracy, and spectral statistics of the learned maps.

\subsection{Loss Function and Evaluation}
Given logits $\ell(x)\in\mathbb{R}^{p}$ and label $y$, the training objective is the average cross-entropy
\[
\mathcal{L}= -\frac{1}{N}\sum_{i=1}^N \log \softmax(\ell(x_i))_{y_i}.
\]
For the spectral diagnostics, we compute the stable rank of the relevant learned linear map or feature matrix, depending on the experiment. Test accuracy is measured in the standard way as the fraction of correctly classified test examples.

\subsection{Hardware and Software Setup}
The experiments were implemented in PyTorch and run on a CUDA-enabled GPU. Random seeds were fixed whenever possible for reproducibility, and all curves shown in the paper were generated from the same codebase and train/test partition.

\subsection{Summary of Hyperparameters}
The main hyperparameters used in the empirical figures are summarized in Table~\ref{tab:hyperparameters}.

\begin{table}[htbp]
\centering
\small
\setlength{\belowcaptionskip}{6pt}
\caption{Main hyperparameters for the modular-addition ReLU MLP experiments.}
\label{tab:hyperparameters}
\begin{tabular}{@{}ll@{}}
\toprule
\rowcolor{gray!12}
\textbf{Hyperparameter} & \textbf{Value}\\
\midrule
Learning rate & $1\times 10^{-3}$\\
Weight decay & $\{0.6,0.7,0.8,0.9,1.0\}$\\
Betas & $(0.9,0.98)$\\
Training epochs & 15{,}000\\
Batch size & Full batch\\
Embedding dimension & 128\\
\bottomrule
\end{tabular}

\end{table}

\section{Fast-Clock Proofs}\label{app:fast-clock-proofs}
In this appendix we give the details for the fast classifier clock. The proofs distinguish margin-value estimates from optimizer convergence estimates; in particular, no PL inequality is inferred from a margin condition alone.

\subsection{Proof of Lemma~\ref{lem:exp-decay}}\label{appendix:proof:lem:exp-decay}
\begin{proof}
For one sample $a$, the gap form of the loss is
\[
\ell_a(A)=\log\left(1+\sum_{m\ne y_a}e^{-\Gamma_{a,m}(A)}\right).
\]
If $\Gamma_{a,m}(A)\ge\gamma$ for every incorrect class, then
\[
\sum_{m\ne y_a}e^{-\Gamma_{a,m}(A)}\le (K-1)e^{-\gamma}.
\]
Therefore
\[
\ell_a(A)\le\log\bigl(1+(K-1)e^{-\gamma}\bigr)\le (K-1)e^{-\gamma},
\]
where the last step uses $\log(1+u)\le u$. Averaging over $a=1,\ldots,M$ gives the same upper bound for $\mathcal L_{\mathrm{CE}}(A)$.
\end{proof}

\subsection{Proof of Lemma~\ref{lem:lipschitz}}\label{appendix:proof:lem:lipschitz}
\begin{proof}
For a single sample with feature vector $x$, write the logits as $z=Ax$ and the softmax vector as $p=\softmax(z)$. The Hessian of the sample loss with respect to the logits is
\[
H_z=\Diag(p)-pp^\top.
\]
For any vector $v\in\mathbb{R}^K$,
\[
v^\top H_zv=\sum_{k=1}^Kp_kv_k^2-\left(\sum_{k=1}^Kp_kv_k\right)^2
=\mathrm{Var}_{k\sim p}(v_k).
\]
The variance of a scalar variable supported on the finite set $\{v_k\}$ is at most one quarter of its squared range, and the squared range is at most $2\|v\|_2^2$. Hence
\[
v^\top H_zv\le \frac12\|v\|_2^2,
\qquad\text{so}\qquad
\|H_z\|_{\mathrm{op}}\le \frac12.
\]
The map $A\mapsto Ax$ has operator norm $\|x\|_2$ from Frobenius norm to Euclidean norm. Therefore the single-sample Hessian with respect to $A$ has operator norm at most $\|x\|_2^2/2\le C^2/2$. Averaging over samples preserves this upper bound, so $\mathcal L_{\mathrm{CE}}$ is $L$-smooth with $L\le C^2/2$.
\end{proof}

\subsection{Proof of Lemma~\ref{lem:gradient-lb}}\label{appendix:proof:lem:gradient-lb}
\begin{proof}
It suffices to give a one-dimensional counterexample. Consider binary logistic loss as a function of the margin $z$,
\[
\ell(z)=\log(1+e^{-z}).
\]
Then
\[
\ell'(z)=-\frac{1}{1+e^{z}}.
\]
As $z\to\infty$, we have $\ell(z)\sim e^{-z}$ and $|\ell'(z)|^2\sim e^{-2z}$. Thus
\[
\frac{|\ell'(z)|^2}{\ell(z)}\sim e^{-z}\to0.
\]
No positive constant $\beta$ can make $|\ell'(z)|^2\ge \beta\ell(z)$ hold uniformly over all sufficiently large margins. Embedding this binary example into any multiclass problem proves the stated impossibility.
\end{proof}

\subsection{Proof of Theorem~\ref{thm:main-exp-convergence}}\label{appendix:proof:thm:main-exp-convergence}
\begin{proof}
Assumption~\ref{assump:margin-persist} gives
\[
\Gamma_{a,m}(A_t)\ge \gamma_0+\nu(t-t_0)
\]
for every training sample and incorrect class. Applying Lemma~\ref{lem:exp-decay} with $\gamma=\gamma_0+\nu(t-t_0)$ yields
\[
\mathcal L_{\mathrm{CE}}(A_t)
\le
(K-1)e^{-\gamma_0-\nu(t-t_0)}.
\]
To make this upper bound at most $\varepsilon$, it is enough that
\[
(K-1)e^{-\gamma_0-\nu(t-t_0)}\le\varepsilon.
\]
Solving for $t$ gives
\[
t\ge t_0+\frac{1}{\nu}\log\!\left(\frac{(K-1)e^{-\gamma_0}}{\varepsilon}\right).
\]
This proves the displayed clock bound.
\end{proof}

\subsection{Proof of Proposition~\ref{thm:single-step-decay}}\label{app:proof:prop:tail-contraction}
\begin{proof}
Iterating the assumed one-step contraction gives
\[
\mathcal L_{\mathrm{CE}}(A_t)
\le
(1-\mu)^{t-t_0}\mathcal L_{\mathrm{CE}}(A_{t_0}).
\]
The condition $\mathcal L_{\mathrm{CE}}(A_t)\le\varepsilon$ is satisfied whenever
\[
(t-t_0)\log\!\left(\frac{1}{1-\mu}\right)
\ge
\log\!\left(\frac{\mathcal L_{\mathrm{CE}}(A_{t_0})}{\varepsilon}\right).
\]
Since $\mu$ is fixed, this is $O(\log(1/\varepsilon))$.
\end{proof}

\section{Spectral and KL Tools for the Slow Clock}\label{app:slow-clock-proofs}
This appendix records the spectral and KL calculations behind the representation clock.

\subsection{Proof of Proposition~\ref{thm:gradient-projection}}\label{app:proof:prop:gradient-projection}
\begin{proof}
On a smooth spectral stratum with $s_i(t)>0$ simple, the singular vectors can be chosen differentiably except on a set of measure zero. The standard perturbation identity for a simple singular value gives
\[
\dot s_i(t)=\langle u_i(t)v_i(t)^\top,\dot A(t)\rangle.
\]
On the same stratum, the spectral penalty has gradient
\[
\nabla\Psi_p(A(t))=\lambda p\,U(t)\mathrm{diag}(s_1(t)^{p-1},\ldots,s_r(t)^{p-1})V(t)^\top.
\]
The gradient flow for the regularized energy is
\[
\dot A(t)=-\nabla\mathcal L_{\mathrm{CE}}(A(t))-\nabla\Psi_p(A(t)).
\]
Projecting this identity onto $u_i(t)v_i(t)^\top$ yields
\[
\dot s_i(t)
=
-\langle\nabla\mathcal L_{\mathrm{CE}}(A(t)),u_i(t)v_i(t)^\top\rangle
-\lambda p\,s_i(t)^{p-1}.
\]
This is the claimed formula after defining the data force $g_i(t)$.
\end{proof}

\subsection{Proof of Theorem~\ref{thm:stable-rank}}\label{app:proof:thm:stable-rank}
\begin{proof}
Let the nonzero singular values of $A_\star$ be $s_1\ge\cdots\ge s_{r_+}>0$. Then
\[
\SRank(A_\star)=\frac{\displaystyle\sum_{i=1}^{r_+}s_i^2}{s_1^2}.
\]
Since each $s_i^2\le s_1^2$, the numerator is at most $r_+s_1^2$, and therefore $\SRank(A_\star)\le r_+$. The lower bound $\SRank(A_\star)\ge1$ holds whenever $A_\star\ne0$ because the numerator contains $s_1^2$.
\end{proof}

\subsection{Stratified Spectral Calculus and Subgradient Flow}\label{app:moreau-yosida}
The Schatten-$p$ penalty with $0<p<1$ is smooth only away from zero singular values. The paper therefore uses a stratum-wise calculation and a limiting-subgradient interpretation at late times.

\begin{theorem}[Local $C^1$ spectral gradient]\label{thm:spectral-C1}
Fix $A_0$ whose nonzero singular values stay positive in a neighborhood $\mathcal U$ and whose rank is constant on that neighborhood. On each smooth stratum inside $\mathcal U$, the map $A\mapsto\Psi_p(A)=\lambda\|A\|_{S_p}^p$ is $C^1$ and
\[
\nabla\Psi_p(A)
=
\lambda p\,U\mathrm{diag}(s_1^{p-1},\ldots,s_r^{p-1})V^\top,
\]
where $A=U\mathrm{diag}(s_1,\ldots,s_r)V^\top$ on the stratum.
\end{theorem}

\begin{proof}
Work on one fixed-rank smooth stratum on which the positive singular values remain separated from zero. Let $A=U\Sigma V^\top$ be a compact singular-value decomposition with $\Sigma=\diag(s_1,\ldots,s_r)$ and $s_i>0$. On this stratum, the Schatten term can be written as
\[
\Psi_p(A)=\lambda\operatorname{tr}\bigl((A^\top A)^{p/2}\bigr)
\]
after restricting to the positive spectral subspace. For a tangent perturbation $H$, the differential of $A^\top A$ is $A^\top H+H^\top A$. The trace functional calculus for the positive matrix $A^\top A$ gives
\[
d\Psi_p(A)[H]
=
\frac{\lambda p}{2}\operatorname{tr}\bigl((A^\top A)^{p/2-1}(A^\top H+H^\top A)\bigr).
\]
Substituting $A=U\Sigma V^\top$ and using cyclicity of trace yields
\[
d\Psi_p(A)[H]
=
\lambda p\operatorname{tr}\bigl(V\Sigma^{p-1}U^\top H\bigr)
=
\bigl\langle \lambda p\,U\diag(s_1^{p-1},\ldots,s_r^{p-1})V^\top,H\bigr\rangle.
\]
Therefore the displayed matrix is the Frobenius gradient on the stratum. If some positive singular values have equal multiplicity, the same expression is invariant under rotations inside the equal-singular-value subspace, so the gradient formula is well defined on each smooth stratum.
\end{proof}

\subsection{Proof of Lemma~\ref{lem:dissipation}}\label{app:proof:lem:dissipation}
\begin{proof}
On any smooth spectral interval, the trajectory satisfies $\dot A(t)=-\nabla\mathcal E(A(t))$. Therefore
\[
\frac{d}{dt}\mathcal E(A(t))
=
\langle\nabla\mathcal E(A(t)),\dot A(t)\rangle
=
-\|\nabla\mathcal E(A(t))\|_{\mathrm F}^2\le0.
\]
If only finitely many stratum changes occur after $T_{\mathrm{KL}}$, the same identity holds on each smooth interval and the energy is continuous across the stratum changes. Concatenating the intervals gives monotone dissipation. At nonsmooth instants, the limiting-subgradient formulation replaces $\nabla\mathcal E$ by the minimum-norm limiting subgradient $\partial^0\mathcal E$.
\end{proof}

\subsection{Kurdyka--Lojasiewicz Tail and Polynomial Decay}\label{app:KL-theory}
The main text states the KL geometry as Assumption~\ref{ass:HOF}. Here we show how it yields the polynomial clock.

\subsection{Proof of Theorem~\ref{thm:KL-global}}\label{app:proof:thm:KL-global}
\begin{proof}
Let $\Delta(t)=\mathcal E(A(t))-\mathcal E(A_\star)$. By Lemma~\ref{lem:dissipation} and Assumption~\ref{ass:HOF}, for almost every late time,
\[
\dot\Delta(t)=-\|\partial^0\mathcal E(A(t))\|_{\mathrm F}^2.
\]
The two-sided KL relation~\eqref{eq:two-sided-kl} gives
\[
c_-^2\Delta(t)^{2\theta}
\le
\|\partial^0\mathcal E(A(t))\|_{\mathrm F}^2
\le
c_+^2\Delta(t)^{2\theta}.
\]
Multiplying by $-1$ reverses the inequalities and yields the displayed differential inequality in Theorem~\ref{thm:KL-global}.
\end{proof}

\subsection{Proof of Theorem~\ref{thm:sharp-rate}: Regularizer-Dominated Dynamics and Polynomial Decay}\label{app:regularizer-dynamics}\label{app:proof:thm:sharp-rate}
\begin{proof}
Theorem~\ref{thm:KL-global} gives constants $a,b>0$ such that
\[
-b\Delta(t)^{2\theta}\le \dot\Delta(t)\le -a\Delta(t)^{2\theta}
\]
for all sufficiently large $t$, with $2\theta>1$. Separating variables for the upper inequality gives
\[
\Delta(t)^{1-2\theta}
\ge
\Delta(T_{\mathrm{KL}})^{1-2\theta}+a(2\theta-1)(t-T_{\mathrm{KL}}),
\]
which implies $\Delta(t)=O((t-T_{\mathrm{KL}})^{-1/(2\theta-1)})$. Separating variables for the lower inequality gives the matching lower bound. Hence
\[
\Delta(t)=\Theta\left((t-T_{\mathrm{KL}})^{-1/(2\theta-1)}\right).
\]
If $\mathcal S(A(t))$ is comparable to $\Delta(t)$ on the same neighborhood, then $\mathcal S(A(t))\le\eta$ occurs at time $\Theta(\eta^{-(2\theta-1)})$ after $T_{\mathrm{KL}}$.
\end{proof}

\section{Generalization and ReLU Transfer Proofs}\label{app:gen-relu-proofs}

\subsection{Proof of Theorem~\ref{thm:robust-and-generalization}}\label{app:proof:robust-generalization}
\begin{proof}
Let $C_1,\ldots,C_K$ be the robustness partition associated with the training set $s$. Write $\mu_i=\mathbb P\{z\in C_i\}$ and let $N_i$ be the number of training samples falling in $C_i$. Denote the empirical frequency by $\widehat\mu_i=N_i/n$.

For cells with $N_i>0$, let $\widehat l_i$ be the empirical average of $l(\mathcal A_s,z)$ over training points in $C_i$, and let $\bar l_i$ be the population conditional average over $C_i$. Robustness implies $|\bar l_i-\widehat l_i|\le \epsilon(s)$, because any point in the cell has loss within $\epsilon(s)$ of every training point in the same cell. For empty cells, the contribution of the population cell is bounded by $M\mu_i$, which is already controlled by $M|\mu_i-\widehat\mu_i|$ because $\widehat\mu_i=0$.

Thus, on every realization of the sample,
\[
\bigl|\mathcal L(\mathcal A_s)-l_{\mathrm{emp}}(\mathcal A_s)\bigr|
\le
\epsilon(s)+M\sum_{i=1}^{K}|\mu_i-\widehat\mu_i|.
\]
The remaining step is a multinomial concentration bound. The Bretagnolle--Huber--Carol inequality gives, with probability at least $1-\delta$,
\[
\sum_{i=1}^{K}|\mu_i-\widehat\mu_i|
\le
\sqrt{\frac{2K\ln2+2\ln(1/\delta)}{n}}.
\]
Substituting this estimate into the previous deterministic inequality proves the stated bound.
\end{proof}

\subsection{Proof of Proposition~\ref{prop:volume}}\label{app:proof:prop:volume}
\begin{proof}
For the lower bound, if $N$ radius-$\alpha$ balls cover $B_2^d(R)$, their total volume must be at least the volume of $B_2^d(R)$. Thus
\[
N|\alpha B_2^d(1)|\ge |B_2^d(R)|,
\]
which gives $N\ge (R/\alpha)^d$ by scaling of Euclidean volume.

For the upper bound, take a maximal $\alpha$-separated set in $B_2^d(R)$. The balls of radius $\alpha/2$ around these points are disjoint and lie inside $B_2^d(R+\alpha/2)$. A volume comparison gives a packing number bounded by a universal constant times $(2R/\alpha)^d$ when $\alpha\le R$. Maximality of the packing implies that radius-$\alpha$ balls around the same centers cover $B_2^d(R)$. Taking logarithms gives the stated logarithmic form.
\end{proof}

\subsection{Proof of Theorem~\ref{thm:stable-rank-embedding}}\label{app:thm:stable-rank-embedding}
\begin{proof}
Let the singular values of $A'$ be $s_1\ge s_2\ge\cdots$. By Eckart--Young,
\[
\|A'-A'_{d_\eta}\|_2=s_{d_\eta+1}
\]
when $d_\eta$ is smaller than the rank, and the left-hand side is zero otherwise. Since the singular values are nonincreasing,
\[
(d_\eta+1)s_{d_\eta+1}^2\le \sum_{i\ge 1} s_i^2=\SRank(A')s_1^2.
\]
With $d_\eta=\lceil\SRank(A')/\eta^2\rceil$, this gives $s_{d_\eta+1}\le\eta s_1=\eta\|A'\|_2$. Therefore, for any embedding $z$,
\[
\|(A'-A'_{d_\eta})z\|_2\le\|A'-A'_{d_\eta}\|_2\|z\|_2\le\eta\|A'\|_2\|z\|_2.
\]
The action of $A'_{d_\eta}$ is determined by the top $d_\eta$ right singular directions, which proves the output-effective subspace statement.
\end{proof}

\subsection{Proof of Proposition~\ref{prop:embedding-covering}}\label{app:proof:prop:embedding-covering}
\begin{proof}
The projected set $P_{V_{d_\eta}}\mathcal M_{\mathrm{embed}}$ lies in a $d_\eta$-dimensional Euclidean ball of radius $R$. Applying Proposition~\ref{prop:volume} in dimension $d_\eta$ gives the bound.
\end{proof}

\subsection{Proof of Theorem~\ref{thm:gen-bound-stablerank}}\label{app:proof:thm:gen-bound-stablerank}
\begin{proof}
By Proposition~\ref{prop:embedding-covering}, the projected embedding set has a cover with
\[
K\le\left(\frac{2R}{\alpha}\right)^{d_\eta}
\]
cells. Projected robustness contributes the term $\epsilon$. Theorem~\ref{thm:stable-rank-embedding} gives a logit perturbation bounded by $\eta\|A'\|_2R_z$ for each endpoint of a within-cell comparison; since the loss is $G$-Lipschitz in logits, this contributes at most $2G\eta\|A'\|_2R_z$. Applying Theorem~\ref{thm:robust-and-generalization} with the above $K$ gives
\[
\bigl|\mathcal{L}(f)-\widehat{\mathcal L}_n(f)\bigr|
\le
\epsilon
+2G\eta\|A'\|_2R_z
+M\sqrt{\frac{2K\ln2+2\ln(1/\delta)}{n}}.
\]
Substituting the covering bound for $K$ proves the theorem.
\end{proof}

\subsection{Proofs for the ReLU Reduction and Temporal-Mismatch Statements}\label{app:relu-bridge-proofs}

\subsection{Proof of Proposition~\ref{prop:relu-frozen-gate}}\label{app:proof:prop:relu-frozen-gate}
\begin{proof}
Fix a training sample $x_i$. On the interval $I$, every ReLU gate is represented by a fixed diagonal matrix $D_{\ell,i}^\star$. Replacing each ReLU by this fixed diagonal map and repeatedly substituting through the layers gives
\[
f_\Theta(x_i)=W_{L+1}(t)D_{L,i}^{\star}W_L(t)\cdots D_{1,i}^{\star}W_1(t)x_i+c_i(t),
\]
where $c_i(t)$ is the affine contribution from the biases propagated through the same gates. If all layers below the classifier head are fixed, the active feature vector is fixed and only the final linear map is trained. The objective is then exactly multiclass logistic regression on these fixed features.
\end{proof}

\subsection{Proof of Proposition~\ref{prop:head-dominant-gradient}}\label{app:proof:prop:head-dominant-gradient}
\begin{proof}
For one sample, write
\[
a_x=Wz_x+b,
\qquad h_x=\phi(a_x)=D_xa_x,
\qquad f_\Theta(x)=Uh_x.
\]
The derivative of cross-entropy with respect to the logits is $\delta_x=p_x-e_y$. The chain rule gives
\[
\nabla_U\ell_x=\delta_xh_x^\top.
\]
For the hidden weights,
\[
\nabla_W\ell_x=D_xU^\top\delta_xz_x^\top,
\]
and for the embedding table,
\[
\nabla_E\ell_x=B_x^\top W^\top D_xU^\top\delta_x.
\]
Using $\|D_x\|_2\le1$ and submultiplicativity yields the two upper bounds on $\|\nabla_W\ell_x\|_{\mathrm F}$ and $\|\nabla_E\ell_x\|_{\mathrm F}$. The identity $\|\delta_xh_x^\top\|_{\mathrm F}=\|\delta_x\|_2\|h_x\|_2$ gives the head-gradient formula.
\end{proof}

\subsection{Proof of Theorem~\ref{thm:temporal-mismatch}}\label{app:proof:thm:temporal-mismatch}
\begin{proof}
The classifier-clock estimate is Theorem~\ref{thm:main-exp-convergence}. It gives
\[
T_{\mathrm{cls}}(\varepsilon)=O\left(\log(1/\varepsilon)\right).
\]
The representation-clock estimate is Theorem~\ref{thm:sharp-rate}. If $\mathcal S(A(t))$ is comparable to the KL energy gap, then
\[
T_{\mathrm{rep}}(\eta)-T_{\mathrm{KL}}
=
\Theta\left(\eta^{-(2\theta-1)}\right).
\]
Substituting $\eta\asymp\varepsilon^q$ gives
\[
T_{\mathrm{rep}}(\varepsilon^q)-T_{\mathrm{KL}}
=
\Theta\left(\varepsilon^{-q(2\theta-1)}\right).
\]
Since $q(2\theta-1)>0$, this grows polynomially in $1/\varepsilon$, while the classifier clock grows only logarithmically.
\end{proof}

\subsection{Proof of Corollary~\ref{cor:relu-temporal-mismatch}}\label{app:proof:cor:relu-temporal-mismatch}
\begin{proof}
Assumption (ii) allows Proposition~\ref{prop:relu-frozen-gate} to replace the ReLU network on the training set by an active linear subsystem. Assumption (i) gives the fast classifier clock on this subsystem through Theorem~\ref{thm:main-exp-convergence}. Assumption (iii) gives the polynomial representation clock through Theorem~\ref{thm:sharp-rate}. Combining the two estimates gives the same temporal mismatch as in Theorem~\ref{thm:temporal-mismatch}.
\end{proof}

\section{Supplementary Remarks on the Two-Clock Mechanism}\label{app:technical-clarifications}
This appendix collects secondary calculations and scope checks that support the two-clock mechanism while keeping the main narrative focused. Its purpose is deliberately limited. The main text proves a fast classifier clock under a post-margin tail condition, a slow representation clock under a late-time structural tail condition, and a conditional transfer statement for stable ReLU regions. The material below records auxiliary estimates, empirical estimators, and boundary cases in a form that can be checked independently of the exposition in the main text.

The appendix is organized around the two stopping times and the auxiliary estimates used to interpret them. Subsection~\ref{app:classifier-clock-calculations} records classifier-clock calculations. Subsection~\ref{app:representation-clock-calculations} records representation-clock calculations. Subsection~\ref{app:conditional-relu-more} states the frozen-gate ReLU transfer more explicitly. Subsection~\ref{app:empirical-clock-estimation} describes how the clocks are estimated in finite training runs. Subsection~\ref{app:scope-temporal-mismatch} states the scope of the temporal-mismatch mechanism.

\subsection{Classifier-Clock Calculations}\label{app:classifier-clock-calculations}
The classifier clock is controlled by the softmax tail after the incorrect-class logits have separated from the correct logit. The estimates in this subsection distinguish three facts that are sometimes conflated: a margin gives a small loss value, a PL inequality requires more than a margin, and discrete-time contraction is a separate optimizer property.

\paragraph{Softmax tail estimates}\label{app:softmax-tail-details}
For one sample with correct class $y$, write the incorrect-class gaps as $\Gamma_m=(w_y-w_m)^\top x$ for $m\ne y$. The cross-entropy loss can be written in gap variables as
\[
\ell(A)=\log\left(1+\sum_{m\ne y}e^{-\Gamma_m}\right).
\]
This identity is the source of the classifier-clock estimate: once the incorrect-class gaps grow, the loss tail is governed by the exponentials of these gaps.

The elementary comparison needed in Lemma~\ref{lem:exp-decay} is
\[
\max_{m\ne y}e^{-\Gamma_m}
\le
\sum_{m\ne y}e^{-\Gamma_m}
\le
(K-1)\max_{m\ne y}e^{-\Gamma_m}.
\]
Consequently, a uniform lower bound $\Gamma_m\ge \gamma$ gives
\[
\ell(A)
\le
\log\bigl(1+(K-1)e^{-\gamma}\bigr)
\le
(K-1)e^{-\gamma}.
\]
This is a value estimate. It says that a sufficiently large gap makes the empirical softmax tail small. The optimizer's rate of increasing the gap is a separate question.

\paragraph{No margin-implied PL inequality}
A positive margin is not a uniform Polyak--Lojasiewicz condition for the cross-entropy tail. The obstruction already appears in the binary one-sample reduction. Let $z=(w_1-w_2)^\top x$ and let
\[
\ell(z)=\log(1+e^{-z}).
\]
Then
\[
\ell'(z)=-\frac{1}{1+e^z},
\qquad
\ell''(z)=\frac{e^z}{(1+e^z)^2}.
\]
As $z\to\infty$, the loss, gradient, and curvature satisfy
\[
\ell(z)\asymp e^{-z},
\qquad
|\ell'(z)|^2\asymp e^{-2z},
\qquad
\ell''(z)\asymp e^{-z}.
\]
Thus
\[
\frac{|\ell'(z)|^2}{\ell(z)}\asymp e^{-z}\to0.
\]
No constant $\beta>0$ can make $|\ell'(z)|^2\ge \beta\ell(z)$ hold uniformly in the large-margin tail. This is why Theorem~\ref{thm:main-exp-convergence} assumes either post-margin gap growth or one-step tail contraction.

\paragraph{Discrete-time contraction}\label{app:fast-clock-discrete-continuous}
The continuous-time notation in the main text has a direct discrete analogue. Suppose that, after a time $t_0$, the empirical loss sequence satisfies
\[
\mathcal L_{t+1}\le (1-\mu)\mathcal L_t,
\qquad 0<\mu<1.
\]
Iterating the contraction gives
\[
\mathcal L_t\le (1-\mu)^{t-t_0}\mathcal L_{t_0}.
\]
Since $1-\mu\le e^{-\mu}$, it is sufficient to take
\[
 t-t_0
 \ge
 \frac{1}{\mu}\log\!\left(\frac{\mathcal L_{t_0}}{\varepsilon}\right)
\]
to reach $\mathcal L_t\le\varepsilon$. This is the discrete classifier clock. It has the same logarithmic dependence on $1/\varepsilon$ as the continuous comparison, but it depends on a genuine tail-contraction statement for the training dynamics.

A gap-growth condition gives another route to the same stopping time. If, over a post-margin interval,
\[
\Gamma_{a,m}(A_t)\ge \gamma_0+\nu(t-t_0),
\]
then Lemma~\ref{lem:exp-decay} gives
\[
\mathcal L_{\mathrm{CE}}(A_t)
\le
(K-1)e^{-\gamma_0-\nu(t-t_0)}.
\]
Solving this inequality for $t$ again gives a logarithmic classifier clock. The two formulations are complementary: one is stated directly in the loss tail, while the other is stated in the logit gaps.

\paragraph{Multiclass smoothness}\label{app:smoothness-constant-detail}
For completeness, we record the smoothness constant used in Lemma~\ref{lem:lipschitz}. For logits $z=Ax$ and softmax vector $p=\softmax(z)$, the Hessian of the single-sample loss with respect to $z$ is
\[
H_z=\diag(p)-pp^\top.
\]
For any $v\in\mathbb R^K$, the corresponding quadratic form is
\[
 v^\top H_z v
 =
 \sum_{k=1}^Kp_k v_k^2-\biggl(\sum_{k=1}^Kp_kv_k\biggr)^2.
\]
This expression is the variance of a scalar random variable taking values $v_1,\ldots,v_K$ under distribution $p$. Its maximum over unit vectors is at most $1/2$, and the binary case with $p=(1/2,1/2)$ attains the bound. Pulling the Hessian back through the map $A\mapsto Ax$ therefore gives a global smoothness bound $C^2/2$ when $\|x\|_2\le C$. The constant is used only for descent estimates; it does not alter the non-PL conclusion above.

\subsection{Representation-Clock Calculations}\label{app:representation-clock-calculations}
The representation clock is driven by a structural quantity after the cross-entropy tail has become small. In the deep-linear surrogate, layerwise weight decay induces a spectral pressure on the end-to-end map. The calculations below state the algebraic surrogate, the local spectral drift, and the comparison estimate that produces the polynomial time scale.

\paragraph{Layerwise decay and the Schatten surrogate}\label{app:layerwise-schatten-details}
Consider an $L$-layer linear factorization
\[
A=W_LW_{L-1}\cdots W_1.
\]
For a fixed end-to-end map $A$, the minimum layerwise squared Frobenius cost over exact factorizations is
\[
\inf_{W_L\cdots W_1=A}\sum_{\ell=1}^L\|W_\ell\|_{\mathrm F}^2
=
L\|A\|_{S_{2/L}}^{2/L}.
\]
This identity explains why layerwise weight decay becomes a nonconvex spectral penalty in the end-to-end model when $L>2$.

To verify the identity, write a compact singular-value decomposition
\[
A=U\diag(s_1,\ldots,s_r)V^\top .
\]
The balanced factorization assigns singular values $s_i^{1/L}$ to each layer along the same singular directions. Its cost is
\[
\sum_{\ell=1}^L\|W_\ell\|_{\mathrm F}^2
=
L\sum_{i=1}^r s_i^{2/L}
=
L\|A\|_{S_{2/L}}^{2/L}.
\]
The matching lower bound follows by applying the arithmetic-geometric mean inequality to the layerwise singular values associated with each active direction. Thus the balanced factorization attains the infimum.

The paper uses this identity as an end-to-end surrogate; a fully factorized training path would require additional control of balance over time. A fully factorized analysis would also need to control imbalance between layers, singular-vector rotation, and rank bottlenecks. These effects are important for a complete dynamics of deep linear networks, but they are separate from the clock comparison isolated here.

\paragraph{Local spectral drift}\label{app:spectral-drift-local}
The singular-value drift formula is a local statement on a smooth spectral stratum. If $s_i(t)>0$ is simple and the corresponding singular vectors vary smoothly, differentiating the regularized flow yields
\[
\dot s_i(t)=g_i(t)-\lambda p s_i(t)^{p-1}.
\]
Here $g_i(t)$ denotes the component of the data force in the $i$th singular direction. The term $-\lambda p s_i(t)^{p-1}$ is stronger for smaller positive singular values when $0<p<1$, which is the basic spectral-selection mechanism.

The formula is local and should be handled with care near nonsmooth spectral events. When singular values collide, individual singular vectors may lose uniqueness. When $s_i=0$ and $0<p<1$, the derivative of $s_i^p$ is singular. Moreover, $g_i(t)$ depends on the full matrix $A(t)$, not just on $s_i(t)$. For these reasons the main text uses the drift calculation to identify the mechanism and then states the actual representation clock through a late-time KL-tail condition.

\paragraph{Polynomial comparison estimate}\label{app:polynomial-clock-comparison}
Let $\Delta(t)$ be the late-time energy gap used to define the representation clock. The polynomial time scale appears when, for all sufficiently large $t$,
\[
-c_2\Delta(t)^{2\theta}
\le
\dot\Delta(t)
\le
-c_1\Delta(t)^{2\theta},
\qquad
0<c_1\le c_2,
\qquad
\theta\in(1/2,1).
\]
Put $r=2\theta-1>0$. Since
\[
\frac{d}{dt}\Delta(t)^{-r}
=
-r\Delta(t)^{-r-1}\dot\Delta(t),
\]
and $r+1=2\theta$, the differential inequality becomes
\[
r c_1
\le
\frac{d}{dt}\Delta(t)^{-r}
\le
r c_2.
\]
Integrating from $T$ to $t$ gives
\[
\Delta(T)^{-r}+r c_1(t-T)
\le
\Delta(t)^{-r}
\le
\Delta(T)^{-r}+r c_2(t-T).
\]
After inversion, for large $t$ there exist constants $C_1,C_2>0$ such that
\[
C_1(t-T)^{-1/r}
\le
\Delta(t)
\le
C_2(t-T)^{-1/r}.
\]
Hence
\[
\Delta(t)=\Theta\bigl((t-T)^{-1/(2\theta-1)}\bigr).
\]
If the structural gap $\mathcal S(A(t))$ is comparable to $\Delta(t)$ on the same late-time region, reaching $\mathcal S(A(t))\le\eta$ requires
\[
 t-T=\Theta\bigl(\eta^{-(2\theta-1)}\bigr).
\]
This calculation is the source of the polynomial representation clock. A one-sided KL inequality gives a one-sided rate, while the two-sided sharp tail in Assumption~\ref{ass:HOF} supports the stopping-time comparison in Theorem~\ref{thm:temporal-mismatch}.

\subsection{Conditional ReLU Transfer}\label{app:conditional-relu-more}
The ReLU component of the paper is a local transfer principle: on intervals where the activation pattern on the training set is fixed, the network coincides with an active linear subsystem on those samples.

For a ReLU preactivation $h_i(x;\Theta)$, define the training-set gate by
\[
\sigma_{i,a}(\Theta)=\mathbf 1\{h_i(x_a;\Theta)>0\}.
\]
If all gates $\sigma_{i,a}$ remain fixed on a time interval, each ReLU nonlinearity is replaced on the training set by a fixed diagonal mask. On that interval the network output on the training samples is an affine function of the remaining linear weights. This is the exact algebra used in Proposition~\ref{prop:relu-frozen-gate}.

A deterministic sufficient condition for this stability is a margin to each gate hyperplane. Suppose that at time $T$,
\[
\min_{i,a}|h_i(x_a;\Theta_T)|\ge \rho>0.
\]
If the parameter path satisfies
\[
\sup_{t\in[T,T+\tau]}
\max_{i,a}|h_i(x_a;\Theta_t)-h_i(x_a;\Theta_T)|<\rho,
\]
then no gate changes on $[T,T+\tau]$. In practice, this condition requires a priori control of parameter movement. The paper therefore uses it as a conditional bridge from nonlinear dynamics to the active linear model.

The head-gradient comparison has the same status. If the output is written as
\[
f_{W,E}(a,b)=W\phi_E(a,b),
\]
then the chain rule gives
\[
\nabla_W\mathcal L = G\phi_E^\top,
\qquad
\nabla_E\mathcal L = (D_E\phi_E)^\ast W^\top G,
\]
where $G$ is the derivative of the loss with respect to the logits. The embedding gradient carries the additional downstream factor $W^\top$ and the feature Jacobian $D_E\phi_E$. When these factors are controlled, the identity explains why the classifier head can move on a faster effective clock than the embedding block. The identity alone is algebraic; the time-scale interpretation requires the norm bounds stated in Proposition~\ref{prop:head-dominant-gradient}.

\subsection{Empirical Estimation of the Clocks}\label{app:empirical-clock-estimation}
The mathematical clocks are stopping times. In experiments they are estimated from finite, noisy trajectories, so the empirical version is best interpreted as a windowed diagnostic. The figures in the paper are used in this sense: they show that classifier-side quantities stabilize early, while representation-side quantities continue to evolve over a longer interval.

\paragraph{Classifier-side estimator}\label{app:measuring-clocks}
Given a recorded sequence of training losses $\mathcal L_0,\mathcal L_1,\ldots$, the empirical classifier clock at threshold $\varepsilon$ is
\[
\widehat T_{\mathrm{cls}}(\varepsilon)
=
\min\{t:\mathcal L_t\le \varepsilon\}.
\]
For noisy runs one may require the inequality to persist over a short window. A second classifier-side diagnostic is the empirical minimum gap
\[
\widehat\Gamma_t=
\min_a\min_{m\ne y_a}\bigl(f_{y_a}(x_a;t)-f_m(x_a;t)\bigr).
\]
The gap complements the loss-based stopping time. It is used to check whether the trajectory has entered a post-margin regime in which the loss tail can be compared to exponentials of the gaps.

\paragraph{Representation-side estimator}
The representation clock is attached to a structural threshold. In the theory this threshold may be an energy gap, a stable-rank tail, or another quantity comparable to the late-time KL gap. In finite experiments, a practical spectral statistic is
\[
\widehat S_t
=
\frac{\|A_t\|_{\mathrm F}^2}{\|A_t\|_2^2}.
\]
A more localized spectral-tail statistic at rank $r$ is
\[
\widehat R_t(r)
=
\frac{\displaystyle\sum_{i>r}s_i(A_t)^2}
       {\displaystyle\sum_{i\ge 1}s_i(A_t)^2}.
\]
Either statistic can define an empirical representation clock by the first time it falls below a chosen threshold. The appropriate statistic depends on what structural simplification the task is expected to use.

\paragraph{Windowed comparison}\label{app:figure-reading}
A finite training run rarely displays a clean power law over the whole trajectory. The relevant empirical question is therefore comparative. One asks whether the training loss and classifier-side diagnostics stabilize over an early window while the representation-side diagnostics continue to change over a substantially longer window. If the test curve improves during the latter interval, the run is consistent with the temporal-mismatch mechanism.

This interpretation keeps the experimental claim at the same level as the theory. The deep-linear results prove a conditional separation between a logarithmic classifier clock and a polynomial representation clock. The nonlinear modular-addition figures show trajectories in which fitting is early and structural simplification is late. They are not used as proof that every grokking run satisfies the deep-linear assumptions.

\paragraph{Empirical protocol}\label{app:empirical-protocol}
A clock-based experiment should report aligned time series for the training loss, a classifier-side gap or head-norm statistic, a representation-side structural statistic, and test accuracy. The alignment prevents the test curve from being interpreted in isolation. In a typical run, one first chooses a loss threshold $\varepsilon$ below the visible interpolation plateau, then records the first sustained time at which the training loss stays below that threshold. One then chooses a structural threshold $\eta$ and records the first sustained time at which the representation statistic stabilizes or falls below the threshold.

Three outcomes are especially informative. If the structural clock and the test curve move together while the classifier clock has already stabilized, the run supports the two-clock interpretation. If the test curve improves before the chosen structural statistic changes, the statistic is probably not the right representation variable for that task. If the statistic changes but test accuracy does not improve, the simplification may not be aligned with the task geometry. These cases are not contradictions; they indicate whether the selected structural gap is relevant to the observed generalization.

\paragraph{Persistence windows}
Thresholds in finite training curves should be interpreted with a persistence convention. For a window length $q\ge1$, define the sustained classifier time by
\[
\widehat T_{\mathrm{cls}}^{(q)}(\varepsilon)
=
\min\{t:\mathcal L_{t+j}\le\varepsilon\text{ for }j=0,\ldots,q-1\}.
\]
The corresponding representation-side estimator is
\[
\widehat T_{\mathrm{rep}}^{(q)}(\eta)
=
\min\{t:\widehat S_{t+j}\le\eta\text{ for }j=0,\ldots,q-1\}.
\]
The windowed definitions are practical empirical diagnostics: they avoid declaring a clock complete because of a single noisy iterate. They also make the empirical comparison closer to the theorem statements, which concern behavior over an interval instead of a one-step fluctuation.

\paragraph{Time parametrization}
The theoretical statements are written in a continuous time variable $t$, while the experiments are indexed by optimization steps. When the learning rate is constant, the two parametrizations differ only by a scalar factor. With a schedule $\alpha_0,\alpha_1,\ldots$, a closer comparison is obtained by using the accumulated optimizer time
\[
\tau_n=\sum_{k=0}^{n-1}\alpha_k.
\]
The empirical curves may then be plotted against $\tau_n$ instead of the raw iteration index $n$. This convention is not essential for the qualitative figures in the paper, but it is useful when comparing runs with different learning-rate schedules.

The same point applies to the interpretation of logarithmic and polynomial clocks. A logarithmic classifier clock in optimizer time concerns optimizer time; wall-clock time and epoch count may scale differently under different schedules. It means that, once the post-margin condition is expressed in the same time variable as the dynamics, the target loss level is reached after an interval proportional to $\log(1/\varepsilon)$. The representation clock should be read in the same parametrization.

\paragraph{Accuracy versus loss-tail clocks}
Training accuracy is a useful visual diagnostic, but it is too coarse to define the classifier clock. The zero-one fitting time
\[
\widehat T_{\mathrm{acc}}
=
\min\{t:\text{all training samples are correctly classified at time }t\}
\]
can occur before the cross-entropy tail is small. In a binary problem, the sign of the margin determines accuracy, whereas the value of $\log(1+e^{-z})$ continues to depend on the magnitude of $z$. The paper therefore uses $\widehat T_{\mathrm{cls}}(\varepsilon)$ as the classifier-side stopping time in place of $\widehat T_{\mathrm{acc}}$.

This distinction is relevant for grokking experiments because the visually flat training-accuracy curve often appears earlier than the flattening of the training loss. The two-clock mechanism concerns the interval after the classifier tail has become small, not merely the interval after zero-one training accuracy has reached one. Using the loss tail makes the clock definition compatible with the softmax estimates in Section~3 and avoids mixing a discontinuous accuracy event with a smooth gradient-flow analysis.

A useful diagnostic is the separation ratio
\[
\widehat R_{\mathrm{time}}(\varepsilon,\eta)
=
\frac{\widehat T_{\mathrm{rep}}^{(q)}(\eta)+1}
     {\widehat T_{\mathrm{cls}}^{(q)}(\varepsilon)+1}.
\]
The ratio depends on the units of training time, the learning-rate schedule, and the chosen thresholds. Its role is descriptive: a large ratio indicates that the structural statistic remains active long after the loss tail has crossed its threshold. For comparisons across runs, the more stable object is the ordering and separation of the aligned curves, not the numerical value of the ratio.

\paragraph{Scale and normalization}
When the representation object is an end-to-end matrix, the statistics above are insensitive to some rescalings but not to all parameterizations. Stable rank is scale invariant, whereas a raw Frobenius tail is not. If $A_t$ is multiplied by a scalar, then
\[
\frac{\|A_t\|_{\mathrm F}^2}{\|A_t\|_2^2}
\]
remains unchanged, but an unnormalized tail $\sum_{i>r}s_i(A_t)^2$ changes by the square of that scalar. This is why the paper uses normalized spectral quantities in the empirical discussion whenever the magnitude of the logits is still changing. A normalized statistic better isolates shape change from mere growth in the head or logit scale.

For embedding plots, the same principle applies. A geometric visualization should be interpreted together with a numerical statistic, because two-dimensional projections can make late movement appear more or less dramatic depending on the plotting scale. The clock comparison is therefore based on aligned scalar time series, with visualizations used as supporting evidence for the nature of the structural change.

\paragraph{Alternative structural gaps}\label{app:alternative-structural-gaps}
Stable rank is used in the main text because it is a simple proxy for effective dimension and is directly connected to the spectral surrogate. Other representation clocks are also possible. The same proof template can use another structural gap if that gap is linked to a generalization or robustness statement and is comparable to the late-time energy gap on the region where the KL analysis is applied.

One alternative is a class-separation gap in representation space. Let $z_a(t)$ be the learned representation of sample $a$, and let $\mu_c(t)$ be the empirical center of class $c$. A normalized separation statistic is
\[
\mathcal S_{\mathrm{sep}}(t)
=
\frac{\displaystyle\sum_{a=1}^{M}\|z_a(t)-\mu_{y_a}(t)\|_2^2}
     {\displaystyle\sum_{\substack{c,c'=1\\ c\ne c'}}^{K}
      \|\mu_c(t)-\mu_{c'}(t)\|_2^2+\delta}.
\]
This quantity measures within-class spread relative to between-class separation. It is useful when the learned rule is expressed through class clustering instead of spectral collapse of a single end-to-end matrix.

For modular arithmetic, a more task-specific quantity can be defined using a projection $P_{\mathrm{rule}}$ onto a task-aligned Fourier or cyclic subspace:
\[
\mathcal S_{\mathrm{rule}}(t)
=
\frac{\|(I-P_{\mathrm{rule}})A_t\|_{\mathrm F}^2}
     {\|A_t\|_{\mathrm F}^2+\delta}.
\]
Parseval's identity implies that this ratio measures the fraction of Frobenius energy outside the task-aligned Fourier or cyclic subspace. This statistic can be sharper for modular-addition grokking, but it is less architecture-agnostic than stable rank. A third possibility is a neural-collapse-type simplex error when late training aligns class means with an equiangular geometry. These alternatives emphasize that the representation clock is a role in the argument, not a commitment to one universal metric.

\subsection{Scope of the Temporal-Mismatch Mechanism}\label{app:scope-temporal-mismatch}
The final claim of the paper is a comparison of convergence times, with low rank serving as one mathematically transparent route to generalization. The assumptions below are stated so that each part of the argument can be checked separately.

\paragraph{Rationale for the formulation}\label{app:rationale-formulation}
The fast clock is formulated through gap growth or tail contraction because a margin value bound is weaker than an optimizer convergence statement. The representation clock is formulated through a sharp late-time KL tail because the singular-value drift is local and does not by itself determine a global stopping time. The ReLU statement is conditional because fixed activation masks give an exact active linear subsystem, while entrance into such a region is a separate dynamical question. These choices keep the theorem statements narrower, but they also make the mechanism mathematically testable.

\paragraph{Dependency structure}\label{app:dependency-structure}
Table~\ref{tab:dependency-structure} lists the inputs used by each component of the argument. The table is included to separate mathematical dependencies from empirical interpretations.

\begin{table}[ht]
\centering
\small
\setlength{\belowcaptionskip}{6pt}
\caption{Dependency structure of the two-clock argument. Each row states the mathematical input used by one part of the paper and the conclusion derived from it.}
\label{tab:dependency-structure}
\setlength{\tabcolsep}{4pt}
\begin{tabular}{@{}>{\raggedright\arraybackslash}p{3.2cm}>{\raggedright\arraybackslash}p{5.3cm}>{\raggedright\arraybackslash}p{5.1cm}@{}}
\toprule
\rowcolor{gray!12}
\textbf{Component} & \textbf{Mathematical input} & \textbf{Conclusion}\\
\midrule
Classifier tail & Uniform lower bound on incorrect-class logit gaps & Cross-entropy is bounded by an exponential softmax tail.\\
Fast classifier clock & Post-margin gap growth or one-step tail contraction & The loss reaches level $\varepsilon$ in $O(\log(1/\varepsilon))$ time.\\
Spectral drift & Smooth spectral stratum and positive simple singular value & Small singular values receive a stronger shrinkage term $s^{p-1}$.\\
Slow representation clock & Sharp late-time KL geometry with $\theta>1/2$ & The structural gap closes on a polynomial time scale.\\
ReLU transfer & Stable activation masks on the training set & The nonlinear network is locally an active linear subsystem.\\
Delayed generalization & Low effective dimension plus robustness of logits & Structural simplification can improve the generalization bound after fitting.\\
\bottomrule
\end{tabular}
\end{table}

The table should be read as a map of the proof. A different task may use a different representation statistic, a different form of implicit regularization, or a different local model for the nonlinear network. The present paper isolates the case in which delayed generalization is explained by a temporal separation between fitting and structural simplification.

\paragraph{Compatibility with other mechanisms}
The two-clock formulation is compatible with more specific accounts of representation learning. For modular arithmetic, a Fourier-based description may identify the rule learned by the network. For classification tasks with class means, a neural-collapse description may identify the limiting geometry. For data with symmetries, a group-representation description may identify the invariant subspace. In each case, the role of the representation clock is to measure the time required for the corresponding structure to emerge after the classifier has already fit the training labels.

This distinction is useful because it separates what the paper explains from what a task-specific theory may additionally explain. The present theory explains why a delay can appear when the stopping time for fitting is much shorter than the stopping time for structural simplification. A Fourier, neural-collapse, or group-theoretic analysis would explain which structure is being approached in a particular task. These explanations are not mutually exclusive; they operate at different levels of resolution.

Formally, replacing stable rank by a task-specific structural gap changes the definition of $T_{\mathrm{rep}}$ but not the logical comparison. If $\mathcal S_{\mathrm{task}}(A(t))$ is the relevant structural gap and satisfies a late-time comparison
\[
\mathcal S_{\mathrm{task}}(A(t))\asymp \Delta(t),
\]
then the same KL-tail calculation yields a representation clock for that structure. The main text uses stable rank because it is a general spectral proxy and is naturally connected to layerwise decay, while allowing other structural quantities in task-specific theories.

\paragraph{Empirical checks against the mechanism}
A convincing empirical case for the two-clock mechanism requires more than a single grokking curve. The mechanism would be weakly supported in a run where the training loss, the representation statistic, and the test accuracy all change on the same time scale. It would also be weakly supported if the chosen structural statistic remains essentially constant while test accuracy improves. In that case the statistic may not be measuring the relevant representation variable, or the delay may be caused by a different mechanism.

Conversely, a run is more informative when it shows a stable ordering of events across thresholds: the loss tail crosses small levels first, the representation statistic continues to evolve, and test accuracy improves only after or during this structural movement. The paper's experiments are intended to show this ordering for the studied setting.

\paragraph{Notation for the stopping-time comparison}\label{app:notation-guide}
The notation in Table~\ref{tab:notation-guide} is used throughout the clock comparison. The same end-to-end object may appear as a classifier, a deep-linear map, or an active ReLU subsystem, depending on the part of the argument being invoked.

\begin{table}[ht]
\centering
\small
\setlength{\belowcaptionskip}{6pt}
\caption{Main notation for the two-clock argument.}
\label{tab:notation-guide}
\setlength{\tabcolsep}{4pt}
\begin{tabular}{@{}>{\raggedright\arraybackslash}p{3.0cm}>{\raggedright\arraybackslash}p{10.8cm}@{}}
\toprule
\rowcolor{gray!12}
\textbf{Symbol} & \textbf{Meaning}\\
\midrule
$A$ & Effective linear classifier or end-to-end active linear map.\\
$W_L\cdots W_1$ & Deep linear factorization of the end-to-end map.\\
$\Gamma_{a,m}(A)$ & Logit gap between the correct class of sample $a$ and incorrect class $m$.\\
$\mathcal L_{\mathrm{CE}}$ & Empirical cross-entropy loss.\\
$T_{\mathrm{cls}}(\varepsilon)$ & First time at which the classifier loss tail is at most $\varepsilon$.\\
$\Psi_p(A)$ & Schatten-type spectral penalty induced by layerwise decay in the surrogate model.\\
$\mathcal E(A)$ & Regularized end-to-end energy $\mathcal L_{\mathrm{CE}}(A)+\Psi_p(A)$.\\
$\Delta(t)$ & Energy gap $\mathcal E(A(t))-\mathcal E(A_\star)$.\\
$\theta$ & Late-time KL exponent for the representation clock.\\
$\mathcal S(A)$ & Structural gap used to define the representation clock.\\
$T_{\mathrm{rep}}(\eta)$ & First time at which the structural gap is at most $\eta$.\\
$\SRank(A)$ & Stable rank $\|A\|_{\mathrm F}^2/\|A\|_2^2$.\\
\bottomrule
\end{tabular}
\end{table}

The important distinction is between $T_{\mathrm{cls}}$ and $T_{\mathrm{rep}}$. The first is attached to the empirical loss tail. The second is attached to structural simplification. The paper's grokking interpretation is the separation between these two stopping times.

\paragraph{Composition of the proof blocks}\label{app:expanded-proof-architecture}
The first proof block isolates the classifier tail. For sample $a$, the loss is written as
\[
\ell_a(A)=\log\left(1+\sum_{m\ne y_a}e^{-\Gamma_{a,m}(A)}\right).
\]
This identity contains the classifier-clock mechanism. If the gaps grow after time $t_0$, the loss is bounded by an exponential function of $t-t_0$. The proof uses the inequality $\log(1+u)\le u$ and a uniform lower bound on the gaps. The optimizer-dependent part is not hidden in this identity; it is exactly the post-margin growth or tail-contraction assumption.

The second proof block concerns structural simplification. On a smooth spectral stratum, the regularized flow gives
\[
\dot s_i(t)
=
-\langle\nabla\mathcal L_{\mathrm{CE}}(A(t)),u_i(t)v_i(t)^\top\rangle
-
\lambda p s_i(t)^{p-1}.
\]
This local calculation identifies the spectral-selection force. The global time scale comes from the late-time gap relation
\[
\dot\Delta(t)\asymp -\Delta(t)^{2\theta},
\qquad \theta>1/2.
\]
Integrating the relation gives the polynomial tail for the representation clock. The local drift and the global KL comparison play different roles: one identifies the direction of structural pressure, and the other controls the time needed to reach a structural tolerance.

The third proof block connects structural simplification to delayed generalization. The deterministic argument is a covering-number bound on an effectively lower-dimensional output class, with a spectral-tail error term left explicit. This avoids treating stable rank as equivalent to generalization. The temporal theorem then composes the two rates: the classifier loss reaches a small target on a logarithmic clock, whereas the structural gap can reach its target only on a polynomial clock. The grokking interval is the interval between these two stopping times.

\paragraph{Boundary cases}
Correct classification alone is weaker than the fast classifier clock. Classification is an argmax property, whereas cross-entropy depends on the numerical size of logit gaps. A positive lower bound on a binary margin $z$ controls $\log(1+e^{-z})$ by an exponential tail, but it does not imply uniform geometric decrease under gradient descent. This boundary case is the reason the classifier clock is tied to the loss tail, with training accuracy treated as a coarser event.

Spectral shrinkage alone is weaker than low effective rank. On a smooth stratum the regularizer contributes $-\lambda p s^{p-1}$ to a positive singular value, but a persistent data force can balance this shrinkage and singular-vector rotations can redistribute energy. If the learned rule genuinely requires many comparable directions, the stable rank need not collapse. The representation-clock theorem is therefore stated through a structural gap and a late-time tail condition, with no unconditional rank-collapse claim.

A ReLU network can be locally linear without being globally reducible to a single linear map. If gate switching persists throughout the late-time interval, the frozen-gate reduction is not available on that interval. One may then need to work with a sequence of active regions or with a different nonlinear representation variable. The conditional ReLU corollary applies only after the stated gate-stability assumption is satisfied.

Low effective dimension is also only one ingredient in a generalization argument. It enters the paper through a robustness and covering argument. A low-rank map can collapse the wrong directions, and a high-dimensional representation can generalize if it captures a task-aligned invariant or group structure. The bound in the main text therefore contains a spectral-tail or structural-error term leaving stable rank as one component of the generalization argument.

\paragraph{Why the final claim is temporal}\label{app:time-not-geometry}
A purely geometric statement would say that good generalization appears after the representation reaches a certain structure. The delay itself is a dynamical comparison between two events. The paper compares
\[
T_{\mathrm{cls}}(\varepsilon)
\quad\text{and}\quad
T_{\mathrm{rep}}(\eta).
\]
Under the classifier-tail condition, the first stopping time has logarithmic dependence on $1/\varepsilon$. Under the sharp representation-tail condition, the second has polynomial dependence on $1/\eta$. The grokking interval is the regime
\[
T_{\mathrm{cls}}(\varepsilon)
\ll t \ll
T_{\mathrm{rep}}(\eta),
\]
where the training loss is already small but the structural threshold has not yet been reached.

This is the sense in which continued training after interpolation still matters: it changes the representation geometry, and the convergence law for that change can be slower than the convergence law for the loss tail. Weight decay is important in this mechanism because it can continue to act on the representation after the cross-entropy gradient has become small. In the deep-linear surrogate this action is expressed through a spectral penalty; in nonlinear networks it may act through a more complicated active subsystem.

\paragraph{Length of the intermediate interval}
The temporal-mismatch theorem can be read as a statement about the length of an intermediate training interval. Suppose the fast clock satisfies
\[
T_{\mathrm{cls}}(\varepsilon)
\le
C_0+C_1\log(1/\varepsilon),
\]
while the representation clock satisfies, for small $\eta$,
\[
T_{\mathrm{rep}}(\eta)
\ge
D_0+D_1\eta^{-r},
\qquad r=2\theta-1>0.
\]
Then the interval between fitting and structural convergence has length at least
\[
T_{\mathrm{rep}}(\eta)-T_{\mathrm{cls}}(\varepsilon)
\ge
D_1\eta^{-r}-C_1\log(1/\varepsilon)-(C_0-D_0).
\]
This inequality is a restatement of the two rates, but it makes the grokking interpretation explicit. The observed plateau is explained by the fact that the representation clock can dominate the classifier clock for small structural tolerance.

If the two tolerances are coupled, for example $\eta=\varepsilon^q$ with $q>0$, the comparison becomes
\[
T_{\mathrm{rep}}(\varepsilon^q)-T_{\mathrm{cls}}(\varepsilon)
\ge
D_1\varepsilon^{-qr}-C_1\log(1/\varepsilon)-O(1).
\]
As $\varepsilon\downarrow0$, the polynomial term dominates the logarithmic term. This is the formal version of the statement that the classifier can appear finished while the representation remains dynamically active.

\paragraph{Dependence on constants}
The clock comparison is asymptotic in the target tolerances, but constants matter in finite experiments. A small polynomial constant can make the representation clock look short over a limited training budget, and a large classifier constant can delay interpolation. The theory therefore predicts a mechanism, with numerical thresholds depending on the finite experiment. In experiments, the relevant evidence is not merely that $T_{\mathrm{rep}}>T_{\mathrm{cls}}$ for one chosen pair of thresholds, but that the same ordering persists over a meaningful range of tolerances and diagnostics.

This point also explains why the theory is compatible with runs that do not visibly grok. If the structural tolerance is not reached within the training budget, test accuracy may remain low even after the loss is small. If the structural clock is not much slower than the classifier clock, test improvement may occur without a long plateau. If the task does not require the structural simplification measured by the selected statistic, delayed test accuracy may not align with that statistic. These cases change the empirical appearance but not the mathematical distinction between the two clocks.

\paragraph{Relation to weight decay}
The role of weight decay in the paper is specific. It supplies a continued spectral pressure after the cross-entropy tail has become small. In the deep-linear surrogate, this pressure is visible through the induced Schatten-type penalty. The mechanism does not require the cross-entropy gradient to vanish exactly; it requires the classifier-side objective to become small enough that continued structural motion is better described by the representation clock than by further correction of training labels.

This also avoids a common overstatement. Weight decay is one mathematically tractable source of a slow structural force behind the phenomenon. Other forms of implicit or explicit regularization may produce a similar clock if they continue to simplify the representation after fitting. Such alternatives would change the structural gap and the proof of the slow clock, but not the logical form of the temporal comparison.

The conclusion should therefore be read as a conditional mechanism. Grokking is explained here when the observed delay matches the mathematical separation between a fast classifier clock and a slow representation clock.

\end{document}